%% file: arxiv.tex
\documentclass[11pt]{article}
\input{preamble}
\usepackage{enumitem}
\usepackage[top=1.25in, bottom=1.25in, left=1in, right=1in]{geometry}
\usepackage{multirow}
\usepackage[bottom]{footmisc}
\usepackage{etoolbox}
\newtoggle{apx} 
\togglefalse{apx}
\newtoggle{thesis}
\togglefalse{thesis}
\newtoggle{arxiv}
\toggletrue{arxiv}
\newcommand{\paper}[0]{paper}
\newcommand{\ift}[1]{\iftoggle{thesis}{#1}{}}
\newcommand{\ifarxiv}[2]{\iftoggle{arxiv}{#1}{#2}}
\title{\bf Online Sampling from Log-Concave Distributions}
\author{
Holden Lee{\thanks{Duke University
}
}
\and
Oren Mangoubi\thanks{{Worcester Polytechnic Institute}}
\and
Nisheeth K. Vishnoi\thanks{Yale University}
}
\date{\today\\$\;$ \\(Version 4)\footnote{V1 appeared on February 21, 2019. V2/V3 made minor changes. V4 corrected an error in applying Azuma's inequality; this changes the final bound in the online theorem from $\poly{\prc{\ep^4}}$ to $\poly\prc{\ep^6}$. V4 is the version in NeurIPS 2019 (up to reordering of sections).}}

\begin{document}
\maketitle

\begin{abstract}
Given a sequence of convex functions $f_0, f_1, \ldots, f_T$, we study the problem of sampling from the Gibbs distribution $\pi_t \propto e^{-\sum_{k=0}^tf_k}$ for each epoch $t$ in an {\em online} manner.
Interest in this problem derives from applications in machine learning, Bayesian statistics, and optimization where, rather than obtaining all the observations at once, one constantly acquires new data, and must continuously update the distribution.
 Our main result is an algorithm that generates roughly independent samples from $\pi_t$ for every epoch $t$ and, under mild assumptions, makes  $\mathrm{polylog}(T)$ gradient evaluations per epoch.
 All previous results imply a bound on the number of gradient or function evaluations which is at least linear in $T$. Motivated by real-world applications, we assume that functions are smooth, their associated distributions have a bounded second moment, and their minimizer drifts in a bounded manner, but do not assume they are strongly convex.
  In particular, our assumptions hold for online Bayesian logistic regression, when the data satisfy natural regularity properties, giving a sampling algorithm with updates that are poly-logarithmic in $T$.  In simulations, our algorithm achieves accuracy comparable to an algorithm specialized to logistic regression.
  Key to our algorithm is a novel stochastic gradient Langevin dynamics Markov chain with a carefully designed variance reduction step and constant batch size. 
  Technically, lack of strong convexity is a significant barrier to analysis and, here, our main contribution is a martingale exit time argument that shows our Markov chain remains in a ball of radius roughly poly-logarithmic in $T$ for enough time to reach within $\eps$ of $\pi_t$.
\end{abstract}

\pagebreak

\tableofcontents
\pagebreak

\input{introduction}
\input{our-results}
\input{related}

\input{pf-overview}
\input{online}
\input{logistic}
\input{offline-results}
\input{offline-overview}

\input{offline}
\input{simulation}

\input{conclusions}

\bibliographystyle{alpha}
\bibliography{bib}

\appendix

\input{simple}

\input{hardness}

\input{inequalities}

\end{document}

%% file: preamble.tex
\usepackage{amsthm,amsmath,amssymb,amsfonts} 
\usepackage{epsfig} 
\usepackage{latexsym,nicefrac,bbm}
\usepackage{xspace}
\usepackage{color,fancybox,graphicx,subfigure}
\usepackage{tabularx}
\usepackage{hyperref, color}
\hypersetup{colorlinks=true,citecolor=blue, linkcolor=blue, urlcolor=blue} 
\usepackage{multicol}
\usepackage{cleveref}
\usepackage{stmaryrd}

\usepackage{float}

\usepackage{array}
\usepackage{makecell}
\usepackage[export]{adjustbox}

\renewcommand{\epsilon}{\varepsilon}

\newcommand{\eps}{\varepsilon}

\usepackage{epsfig}
\usepackage{verbatim}

\newtheorem{assumption}{Assumption}

\usepackage{algorithmicx}
\usepackage{algorithm,caption}
\usepackage{algpseudocode}

\usepackage{mhequ}
\def \be{\begin{equs}}
\def \ee{\end{equs}}

\newtheorem{theorem}{Theorem}[section]
\newtheorem{lemma}[theorem]{Lemma}
\newtheorem{definition}[theorem]{Definition}

\newtheorem*{theorem*}{Theorem}
\newtheorem{remark}[theorem]{Remark}

\newtheorem{prb}[theorem]{Problem}

\crefname{theorem}{Theorem}{Theorems}
\crefname{observation}{Observation}{Observations}
\crefname{proposition}{Proposition}{Propositions}
\crefname{claim}{Claim}{Claims}
\crefname{condition}{Condition}{Conditions}
\crefname{example}{Example}{Examples}
\crefname{fact}{Fact}{Facts}
\crefname{lemma}{Lemma}{Lemmas}
\crefname{corollary}{Corollary}{Corollaries}
\crefname{definition}{Definition}{Definitions}
\crefname{remark}{Remark}{Remarks}

\newtheorem{rem}[theorem]{Remark}

\newtheorem{lem}[theorem]{Lemma}
\newtheorem{thm}[theorem]{Theorem}

\newcommand{\fB}[0]{\mathfrak{B}}

\newcommand{\fD}[0]{\mathfrak{D}}

\newcommand{\E}[0]{\mathbb{E}}
\newcommand{\EE}[0]{\mathop{\mathbb E}}

\newcommand{\cI}[0]{\mathcal{I}}

\newcommand{\cL}[0]{\mathcal{L}}

\newcommand{\N}[0]{\mathbb{N}}

\newcommand{\Pj}[0]{\mathbb{P}}

\newcommand{\R}[0]{\mathbb{R}}
\newcommand{\fR}[0]{\mathfrak{R}}
\newcommand{\bS}[0]{\mathbb{S}}

\newcommand{\one}[0]{\mathbbm{1}}

\newcommand{\al}[0]{\alpha}
\newcommand{\ga}[0]{\gamma}
\newcommand{\Ga}[0]{\Gamma}
\newcommand{\de}[0]{\delta}

\newcommand{\ep}[0]{\varepsilon}

\newcommand{\ka}[0]{\kappa}
\newcommand{\la}[0]{\lambda}

\newcommand{\te}[0]{\theta}
\newcommand{\Te}[0]{\Theta}

\newcommand{\Om}[0]{\Omega}
\newcommand{\si}[0]{\sigma}
\newcommand{\ze}[0]{\zeta}

\newcommand{\sub}[0]{\subset}

\newcommand{\iy}[0]{\infty}

\newcommand{\rc}[1]{\frac{1}{#1}}
\newcommand{\prc}[1]{\pa{\rc{#1}}}
\newcommand{\ff}[2]{\left\lfloor\frac{#1}{#2}\right\rfloor}

\newcommand{\fc}[2]{\frac{#1}{#2}}
\newcommand{\sfc}[2]{\sqrt{\frac{#1}{#2}}}
\newcommand{\pf}[2]{\pa{\frac{#1}{#2}}}

\newcommand{\nb}[0]{\nabla}

\newcommand{\ab}[1]{\left| {#1} \right|}
\newcommand{\an}[1]{\left\langle {#1}\right\rangle}
\newcommand{\ba}[1]{\left[ {#1} \right]}
\newcommand{\bc}[1]{\left\{ {#1} \right\}}

\newcommand{\ce}[1]{\left\lceil {#1}\right\rceil}

\newcommand{\pa}[1]{\left( {#1} \right)}

\newcommand{\ve}[1]{\left\Vert {#1}\right\Vert}
\newcommand{\nv}[1]{\frac{#1}{\left\Vert {#1}\right\Vert}}

\newcommand{\ol}[1]{\overline{#1}}

\newcommand{\ub}[2]{\underbrace{#1}_{#2}}

\newcommand{\wt}[1]{\widetilde{#1}}
\newcommand{\wh}[1]{\widehat{#1}}

\newcommand{\step}[1]{\noindent{\underline{Step {#1}:}}}

\newcommand{\amin}{\operatorname{argmin}}

\newcommand{\KL}[0]{\operatorname{KL}}

\newcommand{\poly}{\operatorname{poly}}

\newcommand{\tr}[0]{\operatorname{tr}}

\newcommand{\Var}[0]{\operatorname{Var}}

\newcommand{\Vol}[0]{\text{Vol}}

\providecommand{\cal}[1]{\mathcal{#1}}
\renewcommand{\cal}[1]{\mathcal{#1}}

\newcommand{\pull}[9]{
#1\ar@/_/[ddr]_{#2} \ar@{.>}[rd]^{#3} \ar@/^/[rrd]^{#4} & &\\
& #5\ar[r]^{#6}\ar[d]^{#8} &#7\ar[d]^{#9} \\}

\newcommand{\cmp}[9]{
\xymatrix{
#1 \ar[r]^{#4}{#5} \ar@/_2pc/[rr]^{#8}_{#9} & #2 \ar[r]^{#6}_{#7} & #3
}
}

\newcommand{\ha}[1]{\ar@{^(->}[#1]}
\newcommand{\ls}[1]{\ar@{-}[#1]}
\newcommand{\sj}[1]{\ar@{->>}[#1]}
\newcommand{\aq}[1]{\ar@{=}[#1]}
\newcommand{\acir}[1]{\ar@{}[#1]|-{\textstyle{\circlearrowright}}}
\newcommand{\acil}[1]{\ar@{}[#1]|-{\textstyle{\circlearrowleft}}}
\newcommand{\ard}[1]{\ar@{.>}[#1]}
\newcommand{\mt}[1]{\ar@{|->}[#1]}
\newcommand{\inm}[1]{\ar@{}[#1]|-{\in}}
\newcommand{\inr}{\ar@{}[d]|-{\rotatebox[origin=c]{-90}{$\in$}}}
\newcommand{\inl}{\ar@{}[u]|-{\rotatebox[origin=c]{90}{$\in$}}}

\newcommand{\sumo}[2]{\sum_{#1=1}^{#2}}
\newcommand{\sumz}[2]{\sum_{#1=0}^{#2}}
\newcommand{\prodo}[2]{\prod_{#1=1}^{#2}}

\newcommand{\beq}[1]{\begin{equation}\llabel{#1}}
\newcommand{\eeq}[0]{\end{equation}}
\newcommand{\bal}[0]{\begin{align*}}
\newcommand{\eal}[0]{\end{align*}}
\newcommand{\ban}[0]{\begin{align}}
\newcommand{\ean}[0]{\end{align}}

\newcommand{\fixme}[1]{{\color{red}#1}}
\newcommand{\llabel}[1]{\label{#1}\text{\fixme{\tiny#1}}}

\newcommand{\arxiv}[1]{\url{http://www.arxiv.org/abs/#1}}

\newcommand{\vocab}[1]{\textbf{#1}} 

\allowdisplaybreaks[2]

\DeclareFontFamily{U}{wncy}{}
    \DeclareFontShape{U}{wncy}{m}{n}{<->wncyr10}{}
    \DeclareSymbolFont{mcy}{U}{wncy}{m}{n}
    \DeclareMathSymbol{\Sh}{\mathord}{mcy}{"58}

\newcommand{\smin}{\si_{\min}}

\newcommand{\slb}[0]{\fc{C\sqrt T+ \beta B}{\si_{\min}T +\al}}

\renewcommand{\tilde}[1]{\widetilde{#1}}

%% file: introduction.tex
\section{Introduction}
\label{s:intro}

In this \paper, we study the following online sampling problem: 
\begin{prb}\label{prb:main}
Consider a sequence of convex functions $f_0,f_1,\ldots, f_T:\R^d\to \R$ for some $T\in \mathbb{N}$, and let $\epsilon>0$.
At each epoch $t \in \{1,\ldots, T\}$, the function $f_t$ is given to us, so that we have oracle access to the gradients of the first $t+1$ functions $f_0,f_1,\ldots, f_t$.
The goal at each epoch $t$ is to generate a sample from the distribution $\pi_t(x)\propto e^{-\sumz kt f_k(x)}$ with fixed total-variation (TV) error $\epsilon$. 
The samples at different time steps should be almost independent.
\end{prb}
\noindent
Various versions of this problem have been  considered in the literature, with applications in Bayesian statistics, optimization, and theoretical computer science; see \cite{narayanan2013efficient, doucet2000rao, andrieu2010particle} and references therein.
If $f$ is convex, then a distribution $p\propto e^{-f}$ is logconcave; this captures a large class of useful distributions such as gaussian, exponential, Laplace, Dirichlet, gamma, beta, and chi-squared distributions.
We give some settings where online sampling can be used:

\begin{itemize}[leftmargin=*]
\item \textbf{Online posterior sampling.} 
In Bayesian statistics, the goal is to infer the probability distribution (the \emph{posterior}) of a parameter, based on observations; however, rather than obtaining all the observations at once, one constantly acquires new data, and must continuously update the posterior distribution, rather than only after all data is collected. 
Suppose $\te\sim p_0\propto e^{-f_0}$ for a given \emph{prior distribution}, and samples $y_t$ drawn 
from the conditional distribution $p(\cdot| \te, y_1,\ldots,y_{t-1})$ arrive in a streaming manner. 
%
%
By Bayes's rule,  
letting $p_t(\te) =e^{-f_t(\te)} := p(\te|y_1,\ldots, y_t)$ be the posterior distribution, 
we have the following recursion:
%
$p_t (\te) \propto p_{t-1}(\te) p(y_t | \te, y_1,\ldots, y_{t-1})$. Hence,  $p_t(\te)\propto  e^{-\sumz kt f_k(\te)}$.
%
The goal is to 
sample 
 from $p_t(\te)$ for each $t$. 
%
%
This fits the setting of Problem~\ref{prb:main} if $p_0$ and all updates $p(y_t | \te, y_1,\ldots  y_{t-1})$ are logconcave.

One practical application is online logistic regression; logistic regression is a common model for binary classification.
Another is inference for Gaussian processes, which are used in many Bayesian models because of their flexibility, and where stochstic gradient Langevin algorithms have been applied~\cite{filippone2015enabling}.
A third application 
is latent Dirichlet allocation (LDA), often used for document classification \cite{blei2003latent}. 
As new documents are published, it is desirable to update the distribution of topics without excessive re-computation.\footnote{
Note that LDA requires sampling from non-logconcave distributions. Our algorithm can be used for non-logconcave distributions, but our theoretical guarantees are only for logconcave distributions.
}
%
\item \textbf{Optimization.} 
One online optimization method is to sample a point from the exponential of the (weighted) negative loss (\cite{cesa2006prediction,hazan2007logarithmic}, Lemma 10 in \cite{narayanan2013efficient}). 
There are settings such as online logistic regression where the only known way to achieve optimal regret is a Bayesian sampling approach \cite{foster2018logistic}, with lower bounds known for the naive convex optimization approach~\cite{hazan2014logistic}. 
%
\item \textbf{Reinforcement learning (RL).} 
 Thompson sampling \cite{russo2017tutorial,dumitrascu2018pg} solves RL problems by maximizing the expected reward at each period with respect to a sample from the Bayesian posterior for the environment parameters, reducing it to the online posterior sampling problem.
%
 %
 %
\end{itemize}
%
In all of these applications, because a sample is needed at every epoch $t$,  it is desirable to have a fast online sampling algorithm.
In particular,  the ultimate goal is to design an algorithm for Problem \ref{prb:main} such that  the number of gradient evaluations is almost \emph{constant} at each epoch $t$, so that the computational requirements at each epoch do not increase over time.
This is challenging  because at epoch $t$, one has to incorporate information from \emph{all} $t+1$ functions $f_0,\ldots, f_t$ in roughly $O(1)$ time.
%

Our main contribution is an algorithm for Problem \ref{prb:main}
that computes $\tilde{O}_T(1)$ gradients per epoch, under mild assumptions on the functions\footnote{The subscript $T$ in $\tilde{O}_T$ means that we only show the dependence on the parameters $t,T$, and exclude dependence on non-$T,t$ parameters such as the dimension $d$, sampling accuracy $\epsilon$ and the regularity parameters $C, \mathfrak{D}, L$ which we define in Section \ref{s:asm}.}.
All previous rigorous results (even with comparable assumptions) 
  imply a bound on  the number of gradient or function evaluations which is at least linear in $T$; see Table \ref{table:online}.
 Our assumptions are motivated by real-world considerations and hold in the setting of online Bayesian logistic regression when the data vectors satisfy natural regularity properties.
%
%
%
%

In the offline setting, our result also implies the first algorithm to sample from a $d$-dimensional log-concave distribution $\propto e^{-\sum_{t=1}^T f_t}$ where the $f_t$'s are not assumed strongly convex and 
the total number of gradient evaluations is roughly $T\log(T)+{\rm poly}(d),$ instead of $T\times {\rm poly}(d)$ implied by prior works (Table \ref{table:offline}).  
%

A natural approach to online sampling is to design a Markov chain with the right steady state distribution \cite{narayanan2013efficient,durmus2018analysis,dwivedi2018log,chatterji2018theory}.
The main difficulty is that running a step of a Markov chain that incorporates all previous functions takes time $\Om(t)$ at epoch $t$; all previous algorithms with provable guarantees suffer from this. 
To overcome this, one must use stochasticity --  for example, sample a subset of the previous functions. 
However, this fails because of the large variance of the gradient.  
Our result relies on a stochastic gradient Langevin dynamics (SGLD) Markov chain with a carefully designed variance reduction step and fixed batch size.
%

We emphasize that we do not assume that the functions $f_t$ are strongly convex. This is important for applications such as logistic regression.
Even if the negative log-prior $f_0$ is strongly convex, 
we cannot obtain the same bounds by using existing results on strongly convex $f$, 
because the bounds depend on the condition number of $\sum_{t=0}^T f_t$, which grows as $T$. 
Lack of strong convexity is a technical barrier to analyzing our Markov chain and, here, our main contribution  is a martingale exit time argument that shows that our Markov chain is constrained to a ball of radius roughly $\nicefrac{1}{\sqrt{t}}$ for time that is sufficient for  it to reach  within $\eps$ of $\pi_t$.

%

%% file: our-results.tex
\section{Our algorithm and results}
\subsection{Assumptions}
\label{s:asm}

Denote by $\mathcal{L}(Y)$ the distribution of a random variable $Y$.  For any two probability measures $\mu,\nu$, denote the 2-Wasserstein distance by $W_2(\mu, \nu) := \inf_{(X,Y)\sim \Pi(\mu,\nu)} \sqrt{\mathbb{E}[\|X-Y\|^2]}$, where $\Pi(\mu,\nu)$ denotes the set of all possible couplings of random vectors $(\hat{X},\hat{Y})$ with marginals $\hat{X}\sim \mu$ and $\hat{Y} \sim \nu$.   For every $t\in \{0,\ldots, T\}$, define $F_t:=\sum_{k=0}^t f_k$, and let  $x^\star_t$ be a minimizer of $F_t(x)$ on $\mathbb{R}^d$.  For any $x\in \mathbb{R}^d$, let $\delta_{x}$ be the Dirac delta distribution centered at $x$. We make the following assumptions:

\ift{\nomenclature[3L]{$L$}{Smoothness (Lipschitz constant of gradient) in Assumption~\ref{a:smooth} (Chapter 3)}}
\ift{\nomenclature[3L0]{$L_0$}{Smoothness (Lipschitz constant of gradient) of $f_0$ in Assumption~\ref{a:smooth} (Chapter 3)}}
\begin{assumption}[\textbf{Smoothness/Lipschitz gradient (with constants $L_0,L>0$)}] 
 \label{Assumption:LipschitzG} \label{a:smooth}
For all $1\le t\le T$ and $x,y\in \R^d$, 
$\ve{\nb f_t(y)-\nb f_t(x)}\le L\ve{x-y}$. For $t=0$, $\ve{\nb f_0(y)-\nb f_0(x)}\le L_0\ve{x-y}$. 
 \end{assumption}
We allow $f_0$ to satisfy our assumptions with a different parameter value, since in Bayesian applications $f_0$ models a ``prior" which has different scaling from $f_1,f_2,\ldots f_T$.
\ift{\nomenclature[3A]{$A$}{Exponential concentration in Assumption~\ref{a:conc} and~\ref{a:wass2}, $\Pj_{X\sim \pi_t}(\ve{X-x^\star_t}\ge \fc{\ga}{\sqrt{t+c}})\le Ae^{-k\ga}$ (Chapter 3)}}
\ift{\nomenclature[3k]{$k$}{Exponential concentration in Assumption~\ref{a:conc} and~\ref{a:wass2}, $\Pj_{X\sim \pi_t}(\ve{X-x^\star_t}\ge \fc{\ga}{\sqrt{t+c}})\le Ae^{-k\ga}$ (Chapter 3)}}
\ift{\nomenclature[3C]{$C$}{Bound on second moment from Assumption~\ref{a:conc} and~\ref{a:wass2}, $m_2^{\rc 2}:=\pa{\E_{x\sim \pi_t}\ve{x-x_t^\star}_2^2}^{\rc 2} \le \fc{C}{\sqrt{t+c}}$ for $C= \pa{2+\rc k}\log\pf{A}{k^2}$ (Chapter 3)}}
\ift{\nomenclature[3c]{$c$}{Offset in Assumptions~\ref{a:conc} and~\ref{a:map} (Chapter 3)}}

\begin{assumption}[\textbf{Bounded second moment with exponential concentration (with constants $A,k>0$, $c\ge 0$)}]
 \label{a:conc}\label{a:wass}\label{Assumption:Wass}
For all $0\le t\le T$ and all $s\ge 0$,
 %
 $\Pj_{X\sim \pi_t}(\ve{X-x^\star_t}\ge \nicefrac{s}{\sqrt{t+c}})\le Ae^{-ks}$.
%
\end{assumption}
Note Assumption~\ref{a:conc} implies a bound on the second moment, $m_2^{\nicefrac 12}:=(\E_{x\sim \pi_t}\ve{x-x_t^\star}_2^2)^{\rc 2} \le \nicefrac{C}{\sqrt{t+c}}$ for $C:= \pa{2+\nicefrac{1}{k}}\log(\nicefrac{A}{k^2})$. For conciseness, we write bounds in terms of this parameter $C$.\footnote{
Having a bounded second moment suffices to obtain (weaker) polynomial bounds (by replacing the use of the concentration inequality with Chebyshev's inequality). 
We use this slightly stronger condition because exponential concentration improves the dependence on $\ep$, and is typically satisfied in practice. 
}

\ift{\nomenclature[3D]{$\fD$}{Drift parameter in Assumption~\ref{a:drift} (Chapter 3)}}
\begin{assumption}[\textbf{Drift of mode (with constants $\mathfrak{D} \geq0$, $c\ge 0$)}]
\label{Assumption:MLE}  \label{a:mle}\label{a:map}\label{a:drift}
For all $0\le t,\tau\le T$ such that $\tau \in [t,\max\{2t,1\}]$, $\|x_t^\star - x_\tau^\star\| \leq \nicefrac{\mathfrak{D}}{\sqrt{t+c}}$.
\end{assumption}
\noindent
Assumption~\ref{Assumption:Wass} says that the ``data is informative enough'' -- the current distribution $\pi_t$ (posterior) concentrates near the mode $x^\star_t$ as $t$ increases.  
The $\frac{1}{t}$ decrease in the second moment is what one would expect based on central limit theorems such as the Bernstein-von Mises theorem. 
Assumption~\ref{Assumption:Wass} is a weaker condition than strong convexity:
if the $f_t$'s are $\al$-strongly convex, then $\pi_t(x)\propto e^{-\sumz kt f_k(x)}$
concentrates to within $\nicefrac{\sqrt{d}}{\sqrt{\al(t+1)}}$; however, 
 many distributions satisfy Assumption~\ref{Assumption:Wass} without being strongly log-concave.  For instance, posterior distributions used in Bayesian logistic regression satisfy Assumption~\ref{Assumption:Wass} under natural conditions on the data, but are not strongly log-concave with comparable parameters (Section \ref{sec:Bayesian_summary}). 
%
Hence, together Assumptions~\ref{Assumption:LipschitzG} and~\ref{Assumption:Wass} are a weaker condition than strong convexity and gradient Lipschitzness, the typical assumptions under which the offline algorithm is analyzed.
Similar to the typical assumptions, our assumptions avoid the ``ill-conditioned'' case when the distribution becomes more concentrated in one direction than another as the number of functions $t$ increases.
%
%

Assumption~\ref{Assumption:MLE} is typically satisfied in the setting where the $f_t$'s are iid. This is the case when we observe iid random variables and define functions $f_t$ based on them, as will be the case for our application to Bayesian logistic regression (Problem~\ref{p:log-reg}).
To help with intuition, note that Assumption~\ref{Assumption:MLE} is satisfied for the problem of Gaussian mean estimation: 
the mode is the same as the mean, and the assumption reduces to the fact that a random walk drifts on the order of $\sqrt t$, and hence the mean of the posterior drifts by $O_T (\nicefrac{1}{\sqrt t})$, after $t$ time steps.
We need this assumption because our algorithm uses cached gradients computed $\Te_T(t)$ time steps ago, and in order for the past gradients to be close in value to the gradient at the current point, the points where the gradients were last calculated should be at distance $O_T(\nicefrac{1}{\sqrt t})$ from the current point.
\iftoggle{thesis}{}{We give a simple example where the assumptions hold (Appendix \ref{sec:simple_example}\ifarxiv{}{ of the supplement}).}
%

In Section \ref{sec:Bayesian_summary} we show these assumptions hold for functions arising in online Bayesian logistic regression; unlike previous work on related techniques \cite{nagapetyan2017true, chatterji2018theory}, our assumptions are weak enough to hold in such applications, as they do not require $f_0,\ldots, f_T$  to be strongly convex.

\input{alg}
\subsection{Result in the online setting} \label{sec:results_online}
In this section we give our main result for the online sampling problem; for additional results in the offline sampling problem, see \ifarxiv{Section~\ref{sec:offline}}{Appendix \ref{sec:offline} in the supplement}.

\begin{theorem}[\textbf{Online variance-reduced SGLD}] \label{thm:main_online}\label{thm:os-main}\label{t:main-param}
Suppose that $f_0,\ldots, f_T:\mathbb{R}^d\rightarrow \mathbb{R}$ are (weakly) convex and satisfy Assumptions \ref{Assumption:LipschitzG}-\ref{Assumption:MLE} with $c=\nicefrac{L_0}{L}$.  Let $C=\pa{2+\nicefrac{1}{k}}\log(\nicefrac{A}{k^2})$. Then there exist parameters $b=9$, 
$\eta_0 = \wt \Te\pf{\ep^4}{L^2 \log^6(T) (C+\fD)^2d}$, and $i_{\max}=\wt O\pf{(C+\fD)^2 \log^2(T)}{\eta_0\ep^2}$, 
such that at each epoch $t$, Algorithm 
\ref{alg:OSAGA}
generates an $\epsilon$-approximate independent sample $\mathsf{X}^{t}$ from $\pi_t$.\footnote{See Definition~\ref{df:ind}\ifarxiv{}{ in the supplement} for the formal definition. Necessarily, $\|\mathcal{L}(\mathsf{X}^{t}) - \pi_t\|_{\mathrm{TV}} \leq \epsilon.$}
The total number of gradient evaluations $i_{\max}$ required at each epoch $t$ is polynomial in $d,L, C,\mathfrak{D}, \epsilon^{-1}$ and $\log(T)$. 
Here, $\wt\Te$ and $\wt O$ hide polylogarithmic factors in $d,L,C,\fD,\ep^{-1}$ and $\log(T)$.
\end{theorem}
 %
Note that the dependence of 
 $i_{\max}$ on $\ep$ is $i_{\max} = \wt O_\ep \prc{{\ep}^{6}}$. 
See Section \ref{s:constants}\ifarxiv{}{ in the supplement} for the proof of Theorem \ref{thm:main_online}. Note that the algorithm needs to know the parameters, but bounds are enough.
%

Previous results all imply a bound on the number of gradient or function
evaluations\footnote{In our setting a gradient can be computed in at worst $2d$ function evaluations. In many applications (including logistic regression) gradient evaluation takes the same number of operations as function evaluation.} at each epoch which is at least linear in $T$.
Our result is the first to obtain bounds on the number of gradient evaluations which are poly-logarithmic, rather than linear, in $T$ at each epoch. We are able to do better by exploiting the sum structure of $-\sumz kt f_t$ and the fact that the $\pi_t$ evolve slowly. See Section~\ref{s:rel} for a detailed comparison.
\subsection{Application to Bayesian logistic regression} \label{sec:Bayesian_summary}
Next, we show that Assumptions \ref{Assumption:LipschitzG}-\ref{Assumption:MLE}, and therefore Theorem \ref{thm:os-main}, hold in the setting of online Bayesian logistic regression, when the data satisfy certain regularity properties. 
Logistic regression is a fundamental and widely used model in Bayesian statistics~\cite{albert1993bayesian}. It has served as a model problem for methods in scalable Bayesian inference~\cite{welling2011bayesian,huggins2016coresets,campbell2017automated,campbell2018bayesian}, of which online sampling is one approach. Additionally, sampling from the logistic regression posterior is the key step in the optimal algorithm for online logistic regret minimization~\cite{foster2018logistic}. 

In Bayesian logistic regression, one models the data $(u_t\in \R^d,y_t\in \{-1,1\})$ as follows: there is some unknown $\te_0\in \R^d$ such that given $u_t$ (the ``independent variable"), for all $t \in \{1,\ldots, T\}$ the ``dependent variable'' $y_t$ follows a Bernoulli distribution with ``success'' probability $\phi(u_t^\top\te)$  ($y_t = 1$ with probability  $\phi(u_t^\top\te)$ and $-1$ otherwise) where $\phi(x):=\nicefrac{1}{(1+e^{-x})}$.
The problem we consider is:
\begin{prb}[\bf Bayesian logistic regression] \label{p:log-reg}
Suppose the $y_t$'s are generated from $u_t$'s as Bernoulli random variables with  ``success'' probability $\phi(u_t^\top\te)$. 
 At every epoch $t \in \{1,\ldots, T\}$, after observing $(u_k,y_k)_{k=1}^t$, return a sample from the posterior distribution\footnote{Here we use a Gaussian prior but this can be replaced by any $e^{-f_0}$ where $f_0$ is strongly convex and smooth.}
$\hat{\pi}_t(\te) \propto e^{-\sumz kt \hat{f}_k(\te)}$, %
 where $\hat{f}_0(\te) := e^{-\nicefrac{\al\ve{\te}^2}{2} }$ and $\hat{f}_k(\te) := -\log [\phi(y_k u_k^\top \te)]$.%
\end{prb}
\noindent
We show that Algorithm~\ref{alg:VRSGLD} succeeds for Bayesian logistic regression under reasonable conditions on the data-generating distribution -- namely, that inputs are bounded and we see data in all directions.\footnote{For simplicity, we state the result (Theorem~\ref{thm:main-logistic-conc}) in the case where the input variables $u$ are iid, but note that the result holds more generally (see Lemma~\ref{thm:general-conc}\ifarxiv{}{ in the supplement} 
 for a more general statement of our result).}

\begin{thm}[\textbf{Online Bayesian logistic regression}]\label{thm:main-logistic-conc}
Suppose that for some $\mathfrak{B}, M, \sigma>0$, we have $\ve{\te_0}\le \mathfrak{B}$ and that $u_t\sim P_u$ are iid, where $P_u$ is a distribution satisfying the following: For $u\sim P_u$, \, \,\, (1) $\ve{u}\leq M$ (``bounded'')  and 
(2) 
$\E_u [uu^\top \one_{|u^\top \te_0|\le 2}] \succeq \si I_d$ (``restricted'' covariance matrix is bounded away from 0). 
Then for the functions $\hat{f}_0,\ldots, \hat{f}_T$ in Problem~\ref{p:log-reg}, and any $\epsilon>0$, there exist parameters $L,\log(A),k^{-1},\mathfrak{D}
= \mathrm{poly}(M, \si^{-1}, \alpha, \mathfrak{B}, d, \epsilon^{-1},\log(T))$ such that Assumptions \ref{Assumption:LipschitzG},~\ref{a:wass}, and~\ref{a:mle} 
hold for all $t$ 
with probability at least $1-\epsilon$.
Therefore Alg.~\ref{alg:VRSGLD} 
gives $\epsilon$-approximate samples from $\pi_t$ for $t\in [1,T]$ with  $\mathrm{poly}(M, \si^{-1}, \alpha, \mathfrak{B}, d, \epsilon^{-1},\log(T))$ gradient evaluations at each epoch.
\end{thm}
\noindent
%
In Section~\ref{s:exp} we show that in numerical simulations, our algorithm achieves competitive accuracy with the same runtime compared to an algorithm specialized to logistic regression, the P\'olya-Gamma sampler. 
However, the P\'olya-Gamma sampler has two drawbacks: its running time at each epoch scales linearly as $t$ (our algorithm scales as $\text{polylog}(t)$), and it is unknown whether P\'olya-Gamma attains TV-error $\varepsilon$ in time polynomial in $\frac{1}{\varepsilon}$, $t$, $d$, and other problem parameters.

%% file: alg.tex
\subsection{Algorithm for online sampling}
%
%
At every epoch $t=1,\ldots, T$,
given gradient access to the functions $f_0, \ldots, f_t$, Algorithm~\ref{alg:VRSGLD} generates a point $X^{t}$ approximately distributed according to $\pi_t\propto e^{-\sumz kt f_k(x)}$. It does so by running SAGA-LD (Algorithm~\ref{alg:SAGA}), with step size $\eta_t$ that decreases as the epoch, and a given number of steps $i_{\max}$. Our main Theorem~\ref{thm:os-main} says that for each sample to have fixed TV error $\varepsilon$, at each epoch the number of steps $i_{\max}$ only needs to be poly-logarithmic in $T$.

Algorithm~\ref{alg:SAGA} makes the following update rule at each step for the SGLD Markov chain $X_i$, for a certain choice of stochastic gradient $g_i$, where $\mathbb{E}[g_i]= \sumz kt \nabla f_k(X_i)$:
%
\begin{align} \label{eq:SGLDupdate}
X_{i+1} = X_{i} - \eta_t g_i + \sqrt{2 \eta_t} \xi_i, \qquad
\xi_i\sim N(0,I_d).
\end{align}
%
%
\noindent Key to our algorithm is the construction of the variance reduced stochastic gradient $g_i$.
It is constructed by taking the sum of the cached gradients at previous points in the chain and correcting it with a batch of constant size $b$.

This variance reduction is only effective when the points where the cached gradients were computed stay within $\wt{O}_T(\nicefrac{1}{\sqrt t})$ of the current mode $x_t^\star$. Algorithm~\ref{alg:VRSGLD} ensures that this holds with high probability by resetting to the sample at the previous power of 2 if the sample has drifted too far. 
%

The step size $\eta_t$ is determined by an input parameter $\eta_0>0$. We set $\eta_t=\nicefrac{\eta_0}{t+c}$ for the following reason:
Assumption \ref{Assumption:Wass} says that the variance of the target distribution $\pi_t$ decreases at the rate $\nicefrac{C^2}{t+c}$, and we want to ensure that the variance of each step of Langevin dynamics decreases at roughly the same rate. 
With the step size $\eta_t=\nicefrac{\eta_0}{t+c}$, the Markov chain can travel across a sub-level set containing most of the probability measure of $\pi_t$ in roughly the same number $i_\mathrm{max}= \wt{O}_T(1)$ of steps at each epoch $t$. We will take the acceptance radius to be  $C'=2.5 (C_1+\mathfrak{D})$ where $C_1$ is given by~\eqref{e:main-C1-online}\ifarxiv{}{ in the supplement}, and show that with good probability this choice of $C'$ ensures $\|X^{t-1} - X^{t'}\| \le \nicefrac{4(C_1+\fD)}{\sqrt{t+c}}$ in 
%
 Algorithm~\ref{alg:OSAGA}.  Note that in practice, one need not know the values of the regularity constants in Assumptions \ref{a:smooth}-\ref{a:drift} but can instead use heuristics to “tune” the Markov chain's parameters.

\ift{\nomenclature[3gi]{$g_i$}{Stochastic gradient at step $i$ (Chapter 3)}}
\begin{algorithm}[h!]
\caption{SAGA-LD\label{alg:SAGA}\label{alg:vrsgld-anchor}}
\textbf{Input:} Oracles for $\nabla f_k$ for $k\in[0,t]$,
step size $\eta>0$, batch size $b\in \N$, number of steps $i_{\mathrm{max}}$, initial point $X_0$,
cached gradients $G^k = \nb f_k(u_k)$ for some points $u_k$, and $s = \sumo kt G^k$. 
\textbf{Output:} $X_{i_{\mathrm{max}}}$
\begin{algorithmic}[1]
\For{$i$ from $0$ to $i_{\mathrm{max}}-1$}
\State	 (Sample batch) Sample with replacement a (multi)set $S$ of size $b$ from $\{1,\ldots, t\}$.
\State	 (Calculate gradients) For each $k\in S$, let $G_{\text{new}}^k = \nb f_k(X_{i})$.
\State	 (Variance-reduced gradient estimate) Let $g_i =\nb f_0(X_i) +  s + \fc{t}b \sum_{k\in S} (G_{\text{new}}^k - G^k)$.
\State	 (Langevin step) Let $X_{i+1} = X_{i} - \eta g_i + \sqrt{2 \eta} \xi_i$ where $\xi_i\sim N(0,I)$.
\State	 (Update sum) Update $s\mapsfrom s + \sum_{k\in \text{set}(S)}(G_{\text{new}}^k - G^k)$.
\State	(Update gradients) For each $k\in S$, update $G^k\mapsfrom G_{\text{new}}^k$.
\EndFor
\end{algorithmic}
\end{algorithm}
%

\begin{algorithm}[h!]
\caption{Online SAGA-LD\label{alg:OSAGA}\label{alg:VRSGLD}} 
\textbf{Input:} $T\in \mathbb{N}$ and gradient oracles for functions $f_t: \R^d \rightarrow \mathbb{R}$, for all  $t\in \{0,\ldots, T\}$ , where only the gradient oracles $\nabla f_0,\ldots, \nabla f_t$ are available at epoch $t$, an initial point $\mathsf{X}^{0}\in \R^d$.\\
\textbf{Input:} step size  $\eta_0>0$, batch size $b>0$,  $i_\mathrm{max}>0$, constant offset $c$, acceptance radius $C'$.\\
\textbf{Output:}  At each epoch $t$, a sample $\mathsf{X}^{t}$
\begin{algorithmic}[1]
\State Set $s=0$. \Comment{Initial gradient sum}
\For{epoch $t=1$ to $T$}
\State	 Set $t' = 2^{\lfloor \log_2(t-1) \rfloor}$ if $t>1$, and $t' = 0$ if $t=1$.	\Comment{The previous power of 2}
	\If{$\ve{\mathsf{X}^{t-1}-\mathsf{X}^{t'}}\le \nicefrac{C'}{\sqrt{t+c}}$}
		 $\mathsf{X}^t_0 \mapsfrom \mathsf{X}^{t-1}$ \Comment{If the previous sample hasn't drifted too far, use the previous sample as warm start}
	\Else  
  		 $\,\,\mathsf{X}^t_0 \mapsfrom \mathsf{X}^{t'}$ \Comment{If the previous sample has drifted too far, reset to the sample at time $t'$}
		\EndIf
\State		Set $G_t\mapsfrom \nb f_t(\mathsf{X}^t_0)$
\State	 Set $s\mapsfrom s+G_t$.
\State	 For all gradients $G_k=\nb f_k(u_k)$ which were last updated at time $t/2$, replace them by $\nb f_k(\mathsf{X}^t_0)$ and update $s$ accordingly.
\State	 Draw $i_t$ uniformly from $\{1,\ldots, i_{\max}\}$.
\State	 Run Algorithm~\ref{alg:SAGA} with step size $\nicefrac{\eta_0}{t+c}$, batch size $b$, number of steps $i_{t}$, initial point $\mathsf{X}^{t}_0$, and precomputed gradients $G_k$ with sum $s$. Keep track of when the gradients are updated. 
\State	 Return the output $\mathsf{X}^t = \mathsf{X}^t_{i_{t}}$ of Algorithm~\ref{alg:SAGA}.
\EndFor
\end{algorithmic}
\end{algorithm}

%% file: related.tex
\vspace{-1mm}
\section{Related work}

\label{s:rel}

\paragraph{Online convex optimization.}

Our motivation for studying the online sampling problem comes partly from the successes of online (convex) optimization~\cite{hazan2016introduction}. In online convex optimization, one chooses a point $x_t\in K$ at each step and suffers a loss $f_t(x_t)$, where $K$ is a compact convex set and $f_t:K\to \R$ is a convex function \cite{zinkevich2003online}. The aim is to minimize the regret compared to the best point in hindsight, where $\text{Regret}_T = \sumo tT f_t(x_t)  - \min_{x^*}\sumo tT f_t(x^*)$. The same offline convex optimization algorithms 
such as gradient descent and Newton's method
can be adapted to the online setting~\cite{zinkevich2003online,hazan2007logarithmic}.

%
%
%
%
%

\textbf{Online sampling.} To the best of our knowledge, all previous algorithms with provable guarantees in our setting require computation time that grows polynomially with $t$. %
This is because any Markov chain taking all previous data into account needs $\Omega_T(t)$ gradient (or function) evaluations per step. On the other hand, there are many streaming algorithms that are used in practice which lack provable guarantees, or which rely on properties of the data (such as compressibility~\cite{huggins2016coresets,campbell2017automated}).

The most relevant theoretical work in our direction is~\cite{narayanan2013efficient}. The authors consider a changing log-concave distribution on a convex body, and show that under certain conditions, they can use the previous sample as a warm start and only take a constant number of steps of their Dikin walk chain at each stage.  
They consider the online sampling problem in the more general setting where the distribution is restricted to a convex body. However, \cite{narayanan2013efficient} do not achieve optimal results in our setting, since they do not separately consider the case when $F_t = \sumz kt f_k$ has a sum structure and therefore require $\Omega(t)$ function evaluations at epoch $t$. 
Moreover, they do not consider how concentration properties of the distribution translate into more efficient sampling. 
When the $f_t$ are linear,
they need $O_T(1)$ steps and $O_T(t)$ evaluations per epoch. However, 
in the general convex setting with smooth $f_t$'s, 
they need $O_T(t)$ steps per epoch and $O_T(t^2)$ evaluations per epoch. 
%

There are many other online sampling and other approaches to estimating changing distributions, used in practice.  The \emph{Laplace approximation}, perhaps the simplest, approximates the posterior distribution with a Gaussian \cite{barber2016laplace}; however, most distributions cannot be well-approximated by Gaussians. \emph{Stochastic gradient Langevin dynamics} \cite{welling2011bayesian} can be used in an online setting; however, it suffers from large variance which we address in this work.  The \emph{particle filter} \cite{del2012concentration,giraud2017nonasymptotic} is a general algorithm to track changing distributions. 
Another approach (besides sampling) is \emph{variational inference}, which has also been  considered in an online setting (\cite{wang2011online}, \cite{broderick2013streaming}).

\begin{table}
\begin{center}
\begin{small}
\begin{tabular}{|c|c|c|}
\hline 
Algorithm
&  oracle calls per  &Other assumptions\\

& \hspace{-2.5mm} epoch \hspace{-2.3mm} &

\tabularnewline
\hline
 Online Dikin walk
   &    \multirow{2}{*}{$O_T(T)$}   &  \textbf{Strong convexity} \\  \cite[\S5.1]{narayanan2013efficient} &  &  Bounded ratio of densities
    \tabularnewline
\hline 
Langevin \cite{durmus2018analysis,dwivedi2018log}
 & $O_T(T)$    & --- \tabularnewline
 \cline{1-3}
SGLD \cite{durmus2018analysis}
 & $O_T(T)$     &  --- \tabularnewline
\hline 
\multirow{2}{*}{SAGA-LD \cite{chatterji2018theory}}
 &  \multirow{2}{*}{$O_T(T)$}   \hspace{-2.5mm}
   &  \textbf{Strong convexity} \\ & & Lipschitz Hessian \\
\hline 
CV-ULD \cite{chatterji2018theory}
 &   $O_T(T)$    
 &  \textbf{Strong convexity}\\
\hline 
\multirow{2}{*}{\bf This work}
 &  \multirow{2}{*}{$\mathrm{polylog}(T)$}
 &   \hspace{-6mm} bounded second moment \hspace{-6mm} \\ & & \hspace{-5mm} bounded drift of minimizer  \hspace{-5mm}
\tabularnewline
\hline 
\end{tabular}
\end{small}
\end{center}
\caption{{\small Bounds on the number of gradient (or function) evaluations required by different algorithms to solve the online sampling problem. Lipschitz gradient is assumed for all algorithms.
\cite{narayanan2013efficient} analyzed the online Dikin walk for a different setting where the target has compact support; here we give the result one should obtain for support $\mathbb{R}^d$, where it reduces to the ball walk. Thus it is possible the assumptions we give for the online Dikin walk can be weakened. Note that the number of gradient or function evaluations for the basic Langevin and SGLD algorithms and online Dikin walk depend multiplicatively on $T$ (i.e., $T$$\times$$\mathrm{poly}(d,L, \textrm{other parameters}))$, while the number of evaluations for variance-reduced SGLD methods depend only additively on $T$ (i.e., $T$$+$$\mathrm{poly}(d,L,\textrm{other parameters}))$.}} 
\label{table:online} \label{table:offline}
\end{table}

%
%
%
%

\textbf{Variance reduction techniques.}
Variance reduction techniques for SGLD were initially proposed in \cite{dubey2016variance}, when sampling from a fixed distribution $\pi \propto e^{-\sum_{t=0}^T f_t}$.  
\cite{dubey2016variance} propose two variance-reduced SGLD techniques, CV-ULD and SAGA-LD.  CV-ULD re-computes the full gradient $\nabla F$ at an ``anchor'' point every $r$ steps and updates the gradient at intermediate steps by subsampling the difference in the gradients between the current point and the anchor point.  SAGA-LD, on the other hand, keeps track of when each gradient $\nabla f_t$ was computed, and updates individual gradients with respect to when they were last computed. 
\cite{chatterji2018theory} show that CV-ULD can sample in the offline setting in roughly $T + \nicefrac{d^2}{\epsilon}(\nicefrac{L}{m})^{6} $ gradient evaluations, and that SAGA-LD can sample in $T + T \nicefrac{\sqrt{d}}{\epsilon} (\nicefrac{L}{m})^{\nicefrac{3}{2}} (1+ L_H)$ evaluations, where $L_H$ is the Lipschitz constant of the Hessian of $-\log(\pi)$.\footnote{The bounds of \cite{chatterji2018theory} are given for sampling within a specified Wasserstein error, not TV error.  The bounds we give here are the number of gradient evaluations one would need if one samples with Wasserstein error $\tilde{\epsilon}$ which roughly corresponds to TV error $\epsilon$; 
 roughly, one requires $\tilde{\epsilon} = O(\nicefrac{\epsilon}{\sqrt{T}})$ to sample with TV error $\epsilon$.}

%% file: pf-overview.tex
\section{Proof overview for online problem}
For the online problem, information theoretic constraints require us to use ``information" from at least $\Omega(t)$ gradients to sample with fixed TV error at the $t$'th epoch (see Appendix \ref{sec:Hardness}). 
Thus, to use only $\wt O_T(1)$ gradients at each epoch, we must reuse gradient information from past epochs.
We accomplish this by reusing gradients computed at points in the Markov chain, including points at past epochs. This saves a factor of $T$ over naive SGLD, but only if we can show that these past points in the chain track the distributions' mode, and that our chain stays close to the mode (Lemma~\ref{l:var}\ifarxiv{}{ in supplement}).

The distribution is concentrated to $O_T(\nicefrac{1}{\sqrt t})$ at the $t$th epoch (Assumption~\ref{a:wass}), and we need the Markov chain to stay within $\wt O_T(\nicefrac{1}{\sqrt t})$ of the mode. The bulk of the proof (Lemma~\ref{l:escape}\ifarxiv{}{ in supplement}) is to show that with high probability (w.h.p.) the chain stays within this ball. 
Once we establish that the Markov chain stays close, we combine our bounds with existing results on SGLD from~\cite{durmus2018analysis} to show that we only need $\wt O_T(1)$ steps per epoch (Lemma~\ref{l:induct}). Finally, an induction with careful choice of constants finishes the proof (Theorem~\ref{t:main-param}). 
Details of each of these steps follow.

\textbf{Bounding the variance of the stochastic gradient (see Lemma \ref{lemma:variance}).}
We reduce the variance of our stochastic gradient 
by using the gradient evaluated at a past point $u_k$ and estimating the difference in the gradients between our current point $X^t_i$ and past point $u_k$.  Using the $L$-Lipschitz property (Assumption~\ref{a:smooth}) of the gradients, we show that the variance of this stochastic gradient is bounded by $\fc{t^2L^2}{b}\max_k\ve{X^t_i-u_k}^2$. To obtain this bound, observe that the individual components $\{\nabla f_{k}( X^{t}_{i}) - \nabla f_{k}(u_k)\}_{k\in S}$ of the stochastic gradient $g^t_i$ have variance at most $={t^2}L^2\max_k\ve{X^t_i-u_k}^2$ by the Lipschitz property. Averaging with a batch saves a factor of $b$.
For the number of gradient evaluations to stay nearly constant at each step, %
increasing the batch size is not a viable option to decrease our stochastic gradient's variance.  Rather, showing that $\|X^t_i -u_k\|$  decreases as $\|X^t_i -u_k\| = \wt O_T(\nicefrac{1}{\sqrt t})$, implies the variance of our stochastic gradient decreases at each epoch at the desired rate.

 %
\textbf{Bounding the escape time from a ball where the stochastic gradient has low variance (see Lemma \ref{l:escape}).}
Our main challenge is to bound the distance $\ve{X_i -u_k}$.  Because we do not assume strong convexity, we cannot use proof techniques of past papers analyzing variance-reduced SGLD methods. \cite{chatterji2018theory, nagapetyan2017true} used strong convexity to show that w.h.p., the Markov chain does not travel too far from its initial point, implying a bound on the variance of their stochastic gradients. Unfortunately, many important applications, including logistic regression, lack strong convexity. 

To deal with the lack of strong convexity, we instead use a martingale exit time argument to show that the Markov chain remains inside a ball of radius $r =\wt O_T(\nicefrac{1}{\sqrt t})$ w.h.p. for a large enough time $i_\mathrm{max}$ for the Markov chain to reach a point within TV distance $\epsilon$ of the target distribution.  Towards this end, we would like to bound the distance from the current state of the Markov chain to the mode $\|X_i^t - x_t^\star\|$ by $\wt O_T(\nicefrac{1}{\sqrt t})$, and bound $\|x_t^\star - u_k\|$ by $\wt O_T(\nicefrac{1}{\sqrt t})$. Together, this allows us to bound the distance $\ve{X_i^t - u_k} = O_T(\nicefrac{1}{\sqrt t})$.  We can then use our bound on $\ve{X_i^t - u_k} = \wt O_T(\nicefrac{1}{\sqrt t})$ together with Lemma \ref{lemma:variance} to bound the variance of the stochastic gradient by roughly $\tilde{O}_T(\nicefrac{1}{t})$.

\emph{Bounding $\ve{x_t^\star-u_k}$.}
Since $u_k$ is a point of the Markov chain, possibly at a previous epoch $\tau \leq t$, roughly speaking we can bound this distance inductively by using bounds obtained at the previous epoch $\tau$ (Lemma~\ref{l:induct}). 
Noting that $u_k = X_i^\tau$ for some $i \leq i_{\mathrm{max}}$, we use our bound for $\|u_k - x_\tau^\star\| = O_T(\nicefrac{1}{\sqrt \tau})= O_T(\nicefrac{1}{\sqrt t})$ obtained at the previous epoch  $\tau$,   together with Assumption \ref{a:mle} which says that $\|x_t^\star - x_{\tau}^\star\| =  O_T(\nicefrac{1}{\sqrt t})$, to bound $\|x_t^\star - u_k\|$. 

\emph{Bounding $\ve{X_i^t - x_t^\star}$.}
To bound the distance $\rho_i:= \|X_i^t - x_t^\star\|$ to the mode,  we would like to bound the increase $\rho_{i+1} - \rho_i$ at each step $i$ in the Markov chain.  
\iftoggle{thesis}{We}{Unfortunately, the expected increase in the distance $\|X_i^t - x_t^\star\|$ is much larger when the Markov chain is close to the mode than when it is far away from the mode, making it difficult to get a tight bound on the increase in the distance at each step.  To get around this problem, we instead} use a martingale exit time argument on $\ve{X_i^t - x_t^\star}^2$, the {\em squared} distance from the current state of the Markov chain to the mode.  The advantage in using squared distance is that the expected increase in squared distance due to the Gaussian noise term  $\sqrt{2 \eta_t} \xi_i$ in the Markov chain update rule (Equation \eqref{eq:SGLDupdate}) is the same regardless of the position of the chain, allowing us to obtain tighter bounds on the increase regardless of the Markov chain's current position.
 We then use weak convexity to bound the component of the increase in $\ve{X^t_i-x_t^\star}^2$ that is due to the gradient term $- \eta_t g_i$, and apply Azuma's martingale concentration inequality to bound the exit time from the ball, showing the chain remains at distance of roughly $\wt O_T(\nicefrac{1}{\sqrt t})$ from the mode.
 

%
%
%
%
\textbf{Bounding the TV error (Lemma~\ref{l:induct}).}
We now show that if $u_k$ is close to $x_{\tau}^\star$, then  $\mathsf{X}^{t}$ will be a good sample from $\pi_t$.
More precisely, we show that if at epoch $t$ the Markov chain starts at $X_0^t$ such that $\ve{X_0^t - x_{\tau}^\star}\le \nicefrac{\mathfrak{R}}{\sqrt{t+c}}$ ($\mathfrak{R}$ to be chosen later), then $\ve{\cL(X_{i_{\max}}^t) - \pi_t}_{\mathrm{TV}} \le O(\nicefrac{\ep}{\log_2(T)})$.

To do this, we use two bounds: a bound on the Wasserstein distance between the initial point $X^{t}_{0}$ and the target density $\pi_t$, and a bound on the variance of the stochastic gradient. We then plug the bounds into Corollary 18 of~\cite{durmus2018analysis} (reproduced as Theorem~\ref{t:dmm}\ifarxiv{}{ in the supplementary material}), to show that $i_{\max} = \wt O_{\epsilon, T} (\poly(\nicefrac{1}{\ep}))$ steps per epoch are sufficient to obtain a bound of $\epsilon$ on the TV error.
\textbf{Bounding the number of gradient evaluations at each epoch.}  
Working out constants, we see $i_{\max} = \poly(d,L,C,\fD,  \ep^{-1}, \log(T))$ suffices to obtain TV-error $\ep$ each epoch. A constant batch size suffices, so the total number of evaluations is $O(i_{\max}b) = \poly(d,L,C,\fD,  \ep^{-1}, \log(T))$.

%% file: online.tex
\section{Proof of online theorem (Theorem~\ref{thm:os-main})}

First we formally define what we mean by ``almost independent''.
\begin{definition}\label{df:ind}
We say that $X^1,\ldots, X^T$ are \vocab{$\ep$-approximate independent samples} from probability distributions $\pi_1,\ldots, \pi_T$ if for independent random variables $Y_t\sim \pi_t$, there exists a coupling between $(X^1,\ldots, X^T)$ and $(Y^1,\ldots, Y^T)$ such that for each $t\in [1,T]$, $X^t=Y^t$ with probability $1-\ep$.
\end{definition}

\subsection{Bounding the variance of the stochastic gradient}
\label{s:sg}

We first show that the variance reduction in Algorithm~\ref{alg:OSAGA} reduces the variance from the order of $t^2$ to $t^2 \ve{x-x'}^2$, where $x'$ is a past point. This will be on the order of $t$ if we can ensure $\ve{x-x'}=O_T\prc{\sqrt t}$. Later, we will bound the probability of the bad event that $\ve{x-x'}$ becomes too large. 

\begin{lem}\label{lemma:variance}\label{l:var}
Fix $x\in \R^d$ and $\{u_k\}_{1\le k\le t}$ and let $S$ be a multiset chosen with replacement from $\{1,\ldots, t\}$. Let 
\begin{align}
g^t &= \nb f_0(x) + \ba{\sumo kt \nb f_k(u_k)} + 
\fc tb \sum_{k\in S} [\nb f_k(x) - \nb f_k(u_k)].
\end{align}
Then 
\begin{align}
\E \ba{
\ve{g^t - \sumz kt \nb f_k(x)}^2
}&
\le \fc{t^2}{b} L^2 
\max_k\ve{x-u_k}^2\\
\ve{g^t - \sumz kt \nb f_k(x)}^2 &\le 4t^2L^2\max_k\ve{x-u_k}^2.
\end{align}
\end{lem}
\begin{proof}
Let $V$ be the random variable given by 
\begin{align}
V&= \fc tb \ba{
\pa{\nb f_k(u_k)-\nb f_k(x)} - 
\EE_{k\in [t]}\ba{\nb f_k(u_k)-\nb f_k(x)}
},
\end{align}
where $k\in [t]$ is chosen uniformly at random. Let $V_1,\ldots, V_b$ be independent draws of $V$. Note that the distribution of $\ve{g^t - \sumz kt \nb f_k(x)}^2$ is the same as that of $\ve{\sumo jb V_j}^2$.
Because the $V_j$ are independent,
\begin{align}
&\E\ba{
\ve{g^t - \sumz kt \nb f_k(x)}^2
}
= \E\ba{\ve{\sumo jb V_j}^2}
= \tr \pa{
\E\ba{
\pa{\sumo jb V_j}\pa{\sumo jb V_j}^\top
}}\\
&=\tr \pa{\E\ba{\sumo jb V_jV_j^\top}}
= \sumo jb \E\ba{\tr(V_jV_j^\top)}
= b \E[\ve{V}^2].\label{e:var-sg1}
\end{align}
We calculate
\begin{align}
\E[\ve{V}^2]&= 
%
\fc{t^2}{b^2} \Var_{k\in [t]}\pa{
\nb f_k(u_k)-\nb f_k(x)}
\\
&\le
\fc{t^2}{b^2} \pa{\EE_{k\in [t]} \ba{
\ve{\nb f_k (u_k)-\nb f_k(x)}^2
} %
}\\
&\le \fc{t^2}{b^2}L^2\max_k \ve{x-u_k}^2.
\label{e:var-sg2}
\end{align}
Combining~\eqref{e:var-sg1} and~\eqref{e:var-sg2} gives the first part.


The final part follows because~\eqref{e:var-sg2} implies $\ve{\sumo jb V_j}^2 \le 4t^2L^2\max_k \ve{x-u_k}^2$.
\end{proof}

\subsection{Bounding the escape time from a ball}
\label{s:escape}
\begin{lem}\label{l:escape}
Suppose that the following hold:
\begin{enumerate}
\item
$F:\R^d\to \R$ is convex, differentiable, and $L$-smooth, with a minimizer $x^\star \in \mathbb{R}^d$.
\item
$\ze_i$ is a random variable depending only on $X_0,\ldots, X_i$ such that $\E[\ze_i|X_0,\ldots, X_i]=0$, and whenever $\ve{X_j-x^\star}\le r$ for all $j\le i$, $\ve{\ze_i} \le S$.
\end{enumerate}
Let $X_0$ be such that $\ve{X_0-x^\star}\le r$ and define $X_i$ recursively by
\begin{align}
X_{i+1} &= X_i - \eta g_i + \sqrt{\eta_t}\xi_i\\
\text{where }g_i &= \nb F(X_i) + \ze_i\\
\xi_i&\sim N(0,I_d),
\end{align}
and define the event 
$G:=\{\ve{X_j-x^\star}\le r \, \,  \forall \, 1\le j\le i_{\max}\}$. 
Then for $r^2> \ve{X_0-x^\star}^2
+ i_{\max} [2\eta^2 (S^2 + L^2r^2) + \eta d]$ and $C_\xi\ge \sqrt{2d}$, 
\begin{align}
\Pj(G^c)
 & \le
 i_{\max}\Bigg[
  \exp\pa{-\fc{(r^2 -  \ve{X_0-x^\star}^2
-  i_{\max} [2\eta^2 (S^2 + L^2r^2) + \eta d])^2}{2i_{\max}( 2\eta Sr + 2 \sqrt{\eta} C_\xi (r+\eta S + \eta Lr) + \eta C_\xi^2 )^2}} \\
&\quad
+ \exp\pa{-\fc{C_\xi^2-d}8}
\Bigg].
\end{align}
\end{lem}

\begin{proof}
Note that if $\ve{x-x^\star}\le r$, then because $F$ is $L$-smooth, $\ve{\nb F(x)}\le L\ve{x-x^\star}\le Lr$.
If $\ve{X_{i}-x^\star}\le r$ and $\ve{\ze_i}\le S$, then 
\begin{align}
& \quad 
\ve{X_{i+1}-x^\star}^2-\ve{X_i-x^\star}^2\\ 
&= 
\ve{X_i- x^\star - \eta g_i + \sqrt{\eta}\xi_i }^2 - \ve{X_i-x^\star}^2\\
&= -2\eta\an{g_i,X_i-x^\star} + \eta^2 \ve{g_i}^2
+ 2\sqrt\eta \an{X_i - x^\star-\eta g_i, \xi_i}
+ \eta\ve{\xi_i}^2\\
&= \ub{- 2\eta\an{\nb F_t(X_i) , X_i - x^\star}}{\le 0 \text{ by convexity}}
- 2\eta \an{\ze_i, X_i-x^\star} + \eta^2\ve{g_i}^2
+ 2\sqrt\eta \an{X_i -x^\star- \eta g_i, \xi_i}
+ \eta\ve{\xi_i}^2\\
&\le - 2\eta \an{\ze_i, X_i-x^\star} + 2\eta^2\pa{\ve{\nb F(x_i)}^2 + \ve{\ze_i}^2}
+ 2\sqrt\eta \an{X_i - x^\star - \eta g_i, \xi_i}
+ \eta\ve{\xi_i}^2\\
&\le  - 2\eta \an{\ze_i, X_i-x^\star} + 2\eta^2 ( L^2r^2 + S^2) + 2\sqrt\eta \an{X_i - x^\star - \eta g_i, \xi_i}
+ \eta\ve{\xi_i}^2 \\
&= 2\eta^2(L^2r^2 + S^2) + \eta d 
\ub{- 2\eta \an{\ze_i, X_i-x^\star} + 2\sqrt\eta \an{X_i - x^\star - \eta g_i, \xi_i}
+ \eta(\ve{\xi_i}^2-d)}{(*)}.
\end{align}
Note that (*) has expectation 0 conditioned on $X_0,\ldots, X_i$. To use  Azuma's inequality, we need our random variables to be bounded. Also, recall that we assumed $\ve{X_i-x^\star}$ is bounded above by $r$. Thus, we define a toy Markov chain coupled to $X_i$ as follows. Let $X_0'=X_0$ and 
\begin{align}
X_{i+1}' &=\begin{cases}
X_i' ,&\text{if }\ve{X_i'-x^\star}\ge r\\
X_i' - \eta g_i + \sqrt \eta \xi_i',&\text{otherwise}
\end{cases}\\
\text{where }
g_i&= \nb F(X_i') + \ze_i,\\
\xi_i'&= \min(C_\xi,\ve{\xi_i}) \nv{\xi_i},\\
\xi_i&\sim N(0,I_d).
\end{align}
Then $Y_i':=\ve{X_i' - x^\star}^2 - i [2\eta^2 (S^2 + L^2r^2) + \eta d]$ is a supermartingale with differences upper-bounded by
\begin{align}
Y_{i+1}'-Y_i' &\le \begin{cases}
0, &\ve{X_i'-x^\star}\ge r\\
-2\eta\an{\ze_i, X_i'-x^\star} + 2\sqrt\eta \an{X_i' - x^\star - \eta g_i, \xi_i'} + \eta(\ve{\xi_i}^2-d),&\ve{X_i'-x^\star}<r
\end{cases}\\
&\le 2\eta Sr + 2\sqrt{\eta}(r + \eta(S + Lr)) C_\xi+ \eta(C_\xi^2-d)\\
&\le2\eta Sr + 2 \sqrt{\eta} C_\xi (r+\eta S + \eta Lr) + \eta C_\xi^2.
\end{align}
By Azuma's inequality, for $\la >0$ and for $r^2 >\ve{X_0-x^\star}^2+ i [2\eta^2 (S^2 + L^2r^2) + \eta d]$, 
\begin{align}
&\quad \Pj\pa{
\ve{X_i'-x^\star}^2 - \ve{X_0-x^\star}^2
-  i [2\eta^2 (S^2 + L^2r^2) + \eta d]
>\la 
}\\
&\le \exp\pa{-\fc{\la^2}{2i( 2\eta Sr + 2 \sqrt{\eta} C_\xi (r+\eta S + \eta Lr) + \eta C_\xi^2 )^2}}\\
&\implies \Pj\pa{
\ve{X_i'-x^\star} >r
}\\
&\le \exp\pa{-\fc{(r^2 -  \ve{X_0-x^\star}^2
-  i [2\eta^2 (S^2 + L^2r^2) + \eta d])^2}{2i( 2\eta Sr + 2 \sqrt{\eta} C_\xi (r+\eta S + \eta Lr) + \eta C_\xi^2 )^2}}.
\end{align}

\noindent
If $\ve{X_i-x^\star}\ge r$ for some $i\le i_{\max}$, then either $\ve{X_i'-x^\star}\ge r$ for some $i\le i_{\max}$, or $X_i$ otherwise becomes different from $X_i'$, which happens only when $\xi_i\ge C_\xi$ for some $i\le i_{\max}$. Thus by the Hanson-Wright inequality, since $C_\xi\ge \sqrt{2d}$,
\begin{align}
&\Pj\pa{\cI\le i_{\max}}\\
&\le \sumo i{i_{\max}} \Pj(\ve{X_i'-x^\star}^2> r^2)
 + \sumo i{i_{\max}} \Pj(\ve{\xi_i}> C_\xi)\\
 &\le i_{\max}\Bigg[
  \exp\pa{-\fc{(r^2 -  \ve{X_0-x^\star}^2
-  i_{\max} [2\eta^2 (S^2 + L^2r^2) + \eta d])^2}{2i_{\max}( 2\eta Sr + 2 \sqrt{\eta} C_\xi (r+\eta S + \eta Lr) + \eta C_\xi^2 )^2}} \\
&\quad 
+ \exp\pa{-\fc{C_\xi^2-d}8}
\Bigg].
\end{align}
\end{proof}

\subsection{Bounding the TV error}\label{s:induct}
Lemma~\ref{l:induct} will allow us to carry out the induction step for the proof of the main theorem. 

We will use the following result of \cite{durmus2018analysis}. Note that this result works more generally with non-smooth functions, but we will only consider smooth functions. Their algorithm, Stochastic Proximal Gradient Langevin Dynamics, reduces to SGLD in the smooth case. We will apply this Lemma with our variance-reduced stochastic gradients in Algorithm~\ref{alg:SAGA}.

\begin{lem}[\cite{durmus2018analysis}, Corollary 18]\label{t:dmm}
Suppose that $f:\R^d\to \R$ is convex and $L$-smooth. 
Let $\cal F_i$ be a filtration with $\xi_i$ and $g(x_i)$ defined on $\cal F_i$, and satisfying $\E[g(x_i)|\cal F_{i-1}] = \nb f(x_i)$, $\sup_x\Var [g(x)| \cal F_{i-1}]\le \si^2<\iy$. 
Consider SGLD for $f(x)$ run with step size $\eta$ and stochastic gradient $g(x)$, with initial distribution $\mu_0$ and step size $\eta$; that is,
\begin{align}
x_{i+1} &= x_i - \eta g(x_i) + \sqrt{\eta} \xi_i,& \xi_i\sim N(0,I).  
\end{align}
Let $\mu_n$ denote the distribution of $x_n$ and let $\pi$ be the distribution such that $\pi\propto e^{-f}$. Suppose
\begin{align}
\label{e:eta}
\eta&\le \min\bc{\fc{\ep}{2(Ld + \si^2)}, \rc L},\\
n &\ge \ce{\fc{W_2^2(\mu_0, \pi)}{\eta \ep}}.
\end{align}
Let $\ol\mu = \rc n \sumo kn \mu_k$ be the ``averaged'' distribution. Then $\KL(\ol \mu |\pi)\le \ep$. 
\end{lem}

\begin{remark}
The result in \cite{durmus2018analysis} is stated when $g(x)$ is independent of the history $\cal F_i$, but the proof works when the stochastic gradient is allowed to depend on history, as in SAGA. For SAGA, $\cal F_i$ contains all the information up to time step $i$, including which gradients were replaced at each time step. 

Note \cite{durmus2018analysis} is derived by analogy to online convex optimization.  The optimization guarantees are only given at the point $\bar x$ equal to the average of the $x_t$ (by Jensen's inequality).  For the sampling problem, this corresponds to selecting a point from the averaged distribution $\ol \mu$.
\end{remark}

\ift{\nomenclature[3Gt]{$G_t$}{$G_t = \bc{
\forall s\le t,\forall 0\le i\le i_s, \ve{X^s_i - x_s^\star} \le \fc{\fR}{\sqrt{s+L_0/L}}
}$ (Chapter 3)}}
\ift{\nomenclature[3Ht]{$H_t$}{$H_t =\bc{\forall s\le t \text{ s.t. $s$ is a power of 2 or $s=0$, }\ve{X^s-x_s^\star}\le \fc{C_1}{\sqrt{s+L_0/L}}}$ (Chapter 3)}}
Define the good events
\begin{align}\label{e:G-df}
G_t &= \bc{
\forall s\le t,\forall 0\le i\le i_s, \ve{X^s_i - x_s^\star} \le \fc{\fR}{\sqrt{s+L_0/L}}
},\\
\label{e:H-df}
H_t &= \bc{\forall s\le t \text{ s.t. $s$ is a power of 2 or $s=0$, }\ve{X^s-x_s^\star}\le \fc{C_1}{\sqrt{s+L_0/L}}}.
\end{align}
$G_t$ is the event that the Markov chain never drifts too far from the current mode (which we want, in order to bound the stochastic gradient of SAGA), and $H_t$ is the event that the samples at powers of 2 are close to the respective modes (which we want because we will use them as reset points). Roughly, $G_t^c$ will involve union-bounding over bad events whose probabilities we will set to be $O\pf{\ep}T$ 
 and $H_t^c$ will involve union-bounding over bad events whose probabilities we will set to be $O\pf{\ep}{\log_2(T)}$.  

\begin{lem}[Induction step]\label{l:induct}
Suppose that Assumptions~\ref{a:smooth},~\ref{a:conc}, and~\ref{a:mle} hold with $c=\fc{L_0}{L}$ and $L_0\ge L$.
Let $X_i^\tau$ be obtained by running Algorithm~\ref{alg:OSAGA} with $C'=2.5(C_1+\fD)$, $C_1\ge C$, and $\fR \ge 2(C_1+\fD)$.
Suppose $\eta_t=\fc{\eta_0}{t+L_0/L}$ and $\ep_2>0$ is such that
\be \label{eq:eta_imax_assumption}
\eta_0 \le \fc{\ep_2^2}{Ld + 9L^2(\fR+\fD)^2/b}, \qquad \qquad
i_{\max} \ge \fc{20(C_1+\fD)^2}{\eta_0\ep_2^2}.
\ee
Suppose $\ep_1>0$ is such that for any $\tau\ge 1$,
\begin{align}\label{e:ep1-def}
\Pj\pa{G_\tau | G_{\tau-1}\cap H_{\tau-1}
} &\ge 1-\ep_1.
\end{align}
Suppose $t$ is a power of 2.
Then the following hold.
\begin{enumerate}
\item
For $t<\tau \le 2t$, $\Pj(G_\tau | G_t\cap H_t)\ge 1-(\tau - t) \ep_1$.
\item
Fix $X^s_i$ for $s\le t, 0\le i\le i_{\max}$ such that $G_t\cap H_t$ holds (i.e., condition on the filtration $\cal F_t$ on which the algorithm is defined). Then
\begin{align}
\ve{\cal L(X^\tau) - \pi_\tau}_{TV}&\le (\tau-t)\ep_1 + \ep_2.
\end{align}
\item
We have for $\tau=2t$,
\begin{align}
\Pj\pa{
G_{\tau} \cap H_{\tau} | G_t\cap H_t
} &\ge 1-(t\ep_1 + \ep_2 + Ae^{-kC_1}).
\label{e:induct-lem3}
\end{align}
\end{enumerate}
These also hold in the case $t=0$ and $\tau=1$, when $L_0\ge L$.
\end{lem}
\begin{proof}
Let $F_t(x) = \sumz kt f_k(x)$. 

First, note that $H_{\tau-1}=\cdots=H_t$, because $H_s$ is defined as an intersection of events with indices $\le s$, that are powers of 2. (See~\eqref{e:H-df}.) Moreover, $G_\tau$ is a subset of $G_{\tau-1}$ for each $\tau$, by~\eqref{e:G-df}.

\paragraph{Proof of Statement 1.} 
The first statement holds by induction on $\tau$ and assumption on $\ep_1$. We need to show $P(G_\tau^c|G_t\cap H_t)\leq (\tau-t)\epsilon_1$ by induction. Assuming it is true for $\tau$, we have by the union bound that
\begin{align}
\Pj(G_{\tau+1}^c|G_t,H_t)&\le
\Pj(G_{\tau+1}^c\cap G_\tau| G_t\cap H_t)+\Pj(G_\tau^c|G_t\cap H_t)\\
&\le
\Pj(G_{\tau+1}^c|G_\tau\cap G_t\cap H_t)+\Pj(G_\tau^c|G_t\cap H_t).
\end{align}
Now the event $G_\tau \cap G_t \cap H_t$ is the same as the event $G_\tau\cap H_\tau$, by the previous paragraph. Thus this is $\leq \ep + (\tau-t)\ep$, completing the induction step.

\paragraph{Proof of Statement 2.} For the second statement, note that for $t<\tau \le 2t$,
\begin{align}
\ve{X^\tau_0 - x_\tau^\star} 
&\le \ve{X_0^\tau-X^t} + \ve{X^t - x_t^\star} + \ve{X_t^\star - x_\tau^\star}\\
&\le \fc{2.5(C_1+\fD)}{\sqrt{\tau + L_0/L}}
+ \fc{C_1}{\sqrt{t + L_0/L}}
+ \fc{\fD}{\sqrt{t + L_0/L}}\\
&\le \fc{4(C_1+\fD)}{\sqrt{\tau + L_0/L}}.
\label{e:0-tau}
\end{align}
where in the 2nd inequality we used that 
\begin{enumerate}
\item
Algorithm~\ref{alg:OSAGA} ensures that $\ve{X_0^\tau-X^t}\le \fc{C'}{\sqrt{\tau + L_0/L}}=\fc{2.5(C_1+\fD)}{\sqrt{\tau + L_0/L}}$ (The algorithm resets $X_0^\tau$ to $X^t$ if $\ve{X_0^\tau - X^t}$ is greater than $\fc{C'}{\sqrt{\tau+L_0/L}}$, making the term 0. This is the place where the resetting is used.),
\item the definition of $H_t$, and 
\item the drift assumption (Assumption~\ref{a:mle}).
\end{enumerate}
In the 3rd inequality we used that $\sqrt{t}\ge \sqrt{\tau/2}\ge \sqrt{\tau}/1.5$.

Therefore
\begin{align}
\label{e:0-wass}
W_2^2(\de_{X^\tau_0}, \pi_\tau) 
&\le 2\ve{X^\tau_0-x_\tau^\star}^2 
+ 2W_2^2(\de_{x_\tau} , \pi_\tau)
\le \fc{32(C_1+\fD)^2}{\tau + L_0/L} + \fc{2C^2}{\tau + L_0/L} \le \fc{40(C_1+\fD)^2}{\tau+L_0/L}.
\end{align}
where the second moment bound comes from Assumption~\ref{a:conc} and $C\le C_1$.

Define a toy Markov chain coupled to $X_i^\tau$ as follows. Let $\wt X^s_j=X^s_j$ for $s<\tau$, $\wt X_0^\tau = X_0^\tau$, and 
\begin{align}
\wt X_{i+1}^\tau &=
\begin{cases}
\wt X_i^\tau - \eta g_i^\tau + \sqrt{\eta} \xi_i,&\text{when }\ve{\wt X_j^\tau - x_\tau^\star}\le \fc{\fR}{\sqrt{\tau + L_0/L}}\text{ for all }0\le j\le i\\
\wt X_i^\tau - \eta\nb F_\tau(\wt X_i),&\text{otherwise.}
\end{cases}\label{e:toymc}
\end{align}
where $g_i^\tau$ is the stochastic gradient for $\wt X_i^\tau$ in Algorithm~\ref{alg:SAGA} and $\xi_i\sim N(0,I_d)$. 
By Lemma~\ref{l:var}, the variance of $g_i^\tau$ is at most 
$\fc{\tau^2 L^2}{b}
\max_{(\fc{\tau+1}2,0)\le (s,j)\le (\tau, i)}\ve{\wt X_i^\tau - \wt X_j^s}^2$.
(The ordering on ordered pairs is lexicographic. Note $s>\fc t2$ because Algorithm~\ref{alg:OSAGA} refreshes all gradients that were updated at time $\fc t2$.) If the first case of~\eqref{e:toymc} always holds, we bound (using the condition that $G_t$ holds)
\begin{align}
\ve{\wt X_i^\tau - \wt X_j^s}
&\le \ve{\wt X_i^\tau - x_{\tau}^\star} + \ve{x_{\tau}^\star-x_s^\star} + \ve{x_s^\star - \wt X_j^s}\\
&\le \fc{\fR}{\sqrt{\tau+L_0/L}}
+ \fc{\fD}{\sqrt{s+L_0/L}}
+ \fc{\fR}{\sqrt{s+L_0/L}}\\
& \le \fc{3\fR + 2\fD}{\sqrt{\tau+L_0/L}} < \fc{3(\fR+\fD)}{\sqrt{\tau+L_0/L}}
\label{e:induct-sg-bound0}
\\
\implies
\fc{\tau^2 L^2}{b}
\max_{(\fc{t+1}2,0)\le (s,j)\le (\tau, i)}\ve{\wt X_i^\tau - \wt X_j^s}^2
&\le \fc{9\tau L^2 (\fR+\fD)^2}{b}.\label{e:induct-sg-bound}
\end{align}

\noindent
We can apply Lemma~\ref{t:dmm} with $\ep=2\ep_2^2$, $L\mapsfrom L(\tau + L_0/L)$, %
$\si^2\le \fc{9\tau L^2(\fR+\fD)^2}b$, %
$W_2^2(\mu_0, \pi) \le \fc{40(C_1+\fD)^2}{\tau+L_0/L}$. 
Note that $\eta_\tau\le \fc{\ep_2^2}{(\tau + L_0/L)(Ld + 9L^2(\fR+\fD)^2/b)}\le \fc{\ep_2^2}{(\tau L + L_0)d + 9L^2\tau(\fR+\fD)^2/b}$ does satisfy~\eqref{e:eta}, as $F_\tau = \sumz k{\tau} f_k$ is $(\tau L + L_0)$-smooth by Assumption~\ref{a:smooth}. Let $i\in [i_{\max}]$ be uniform random on $[i_{\max}]$,  and $\wt X^\tau = \wt X^\tau_{i}$; note that the distribution $\wt\mu$ of $\wt X^\tau$ is the mixture distribution of $\wt X_1^\tau,\ldots, \wt X_{i_{\max}}^\tau$. 
 Under the conditions on $\eta, i_{\max}$, by Pinsker's inequality and Lemma \ref{t:dmm},
\begin{align}
 \|\mathcal{L}(\wt{X}^{\tau}) - \pi_\tau \|_{\mathrm{TV}} \leq \sqrt{\rc2\mathrm{KL}(\wt \mu| \pi_\tau)} \leq \ep_2.
\end{align}
Note that under $G_\tau$, $X_i^s = \wt X_i^s$ for all $i\le i_{\max}$ and $s\le \tau$, so 
\begin{align}
 \|\mathcal{L}({X}^{\tau}) - \pi_\tau \|_{\mathrm{TV}} &\leq
\Pj(G_\tau^c | \cal F_t)
+ \|\mathcal{L}(\wt{X}^{\tau}_{i}) -\pi_\tau  \|_{\mathrm{TV}} \leq
 (\tau - t)\ep_1+\ep_2.
 \label{e:induct1}
\end{align}
This shows Statement 2. 

\paragraph{Proof of Statement 3.}
For Statement 3, note that by Assumption~\ref{a:conc},
\begin{align}
\Pj_{X\sim \pi_{2t}}\ba{
\ve{X-x_{2t}^\star}\ge \fc{C_1}{\sqrt{2t+L_0/L}}
}&\le 
Ae^{-kC_1}\label{e:induct2}.
\end{align}
Combining~\eqref{e:induct1} and~\eqref{e:induct2} for $\tau=2t$ gives~\eqref{e:induct-lem3}.

Finally, note that the proof goes through when $t=0$, $\tau=1$.
\end{proof}


\subsection{Setting the constants; Proof of main theorem}
\label{s:constants}

\begin{proof}[Proof of Theorem~\ref{thm:os-main}]
%
We set the parameters $\eta_0, i_{\mathrm{max}}$ of Algorithm~\ref{alg:OSAGA}, as follows:
\begin{align}
\label{e:ep1}
\ep_1&=\fc{\ep}{3T},\\
\label{e:ep2,e:ep} 
\ep_2&= \fc{\ep}{3 \ce{\log_2(T)+1}},\\
\label{e:main-C1-online}
C_1 &=\pa{2+\rc k}\log\pf{A}{\ep_2k^2},\\
\label{e:main-R}
\fR &= 
\fc{10000(C_1+\fD)\sqrt d}{\ep_2} 
{\log\pa{\max\bc{L,C_1+\fD, \rc{\ep_1}}}},
\\
\label{e:main-eta0}
\eta_0 &= \fc{\ep_2^{2}}{2L^2(\fR+\fD)^2},\\
\label{e:main-imax}
i_{\max} &= 
\ce{\fc{20(C_1+\fD)^2}{\eta_0\ep_2^2}}
=\ce{\fc{40L^2(\fR+\fD)^2 (C_1+\fD)^2}{\ep_2^4}}.
\end{align}

\noindent
We can check that $\eta_0 = \wt \Te\pf{\ep^4}{L^2 \log^6(T) (C+\fD)^2d}$, and $i_{\max}=\wt O\pf{(C+\fD)^2 \log^2(T)}{\eta_0\ep^2}$ (where $\wt\Te$ and $\wt O$ hide polylogarithmic dependence on $d,L,C,\fD,\ep^{-1}$ and $\log(T)$, as claimed in Theorem~\ref{thm:os-main}.
The constants have not been optimized.

We will choose parameters and prove by induction that for $t=2^a$, $a\in \N_0$, $t\le T$,
\begin{align}
\label{e:indhyp1}
\Pj(G_{t}\cap H_{t}) &\ge 1 - t\ep_1 - 2(a+1)\ep_2.
\end{align}
We will also show that~\eqref{e:indhyp1} implies that if $t=2^a+b$ for $0< b\le 2^a$,
\begin{align}
\Pj(G_{t}\cap H_{2^a}) &\ge 1 - t\ep_1 - 2(a+1)\ep_2,
\label{e:indhyp1.1}
\\
\label{e:indhyp2}
\ve{\cal L(X_{t}) - \pi_{t}}_{\text{TV}} &\le  t\ep_1 + (2a+3)\ep_2. 
\end{align}
With the values of $\ep_1$ and $\ep_2$,~\eqref{e:indhyp2} gives the theorem, except for the $\ep$-approximate independence of the samples.  
To obtain approximate independence, note that the distribution of $X^t$ conditioned on the filtration $\mathcal{F}_1 \subseteq\cdots\subseteq \mathcal{F}_{t-1}$, where the filtration $\mathcal{F}_\tau$ includes both the random batch $S$ as well as the points in the Markov chain up to time $\tau$, satisfies  $\| (\mathcal{L}(X^t) | F_{t-1}) - \pi_t\|_{\mathrm{TV}} \leq t\ep_1 + (2a+3)\ep_2$.  This implies that the samples $X^1, X^2, \ldots,X^t$ are $\epsilon$-approximately independent with $\epsilon = t\ep_1 + (2a+3)\ep_2$.

Let $\eta_0,\fR$ be constants to be chosen, and for any $t\in \N$, let
\begin{align}
\label{e:eta-t}
\eta_t &= \fc{\eta_0}{{t+L_0/L}},\\
\label{e:r-t}
r_t &= \fc{\fR}{\sqrt{t+L_0/L}},\\
\label{e:s-t}
S_t &= 6 \sqrt t L (\fR+\fD),
\end{align}
We claim that it suffices to choose parameters so that the following hold for each $t$, $1\le t\le T$, and some $C_\xi\ge \sqrt{2d}$:
\begin{align}
\label{e:main1}
  \ep_1 & \ge i_{\max}\Bigg[
  \exp\pa{-\fc{\pa{r_t^2 -  
   \fc{16(C_1+\fD)^2}{t+L_0/L} 
-  i_{\max} [2\eta_t^2 (S_t^2 + L^2t^2 r_t^2) + \eta_t d]}^2}{2i_{\max}( 2\eta_t S_t r_t + 2 \sqrt{\eta_t} C_\xi (r_t+\eta_t S_t + \eta_t L(t+L_0/L) r_t) + \eta_t C_\xi^2 )^2}} 
\\
&\qquad \qquad  +\exp\pa{-\fc{C_\xi^2-d}8}
\Bigg],\\
\eta_0&\le \fc{\ep_2^2}{Ld + 9L^2(\fR+\fD)^2/b},
\label{e:main2}\\
\label{e:main3}
i_{\max} &\ge \fc{20(C_1 + \fD)^2}{\eta_0 {\ep_2}^2},\\
\label{e:main4}
Ae^{-kC_1}&\le \ep_2,\\
\label{e:main5}
C_1 &\ge 
\pa{2+\rc k}\log\pf{A}{\ep_2k^2}.
\end{align}
We first complete the proof assuming that these inequalities hold. Then we show that with the parameter settings in 
\eqref{e:ep1}--\eqref{e:main-imax}, 
these inequalities hold.

Suppose that for some $t<T$ the inequalities \eqref{e:main1}-\eqref{e:main5} hold and the event $G_t\cap H_t$ occurs. We will apply Lemma~\ref{l:escape} 
to the call of the SAGA-LD algorithm in Algorithm~\ref{alg:OSAGA}, at epoch $t+1$ with $F(x) = \sumz s{t+1} f_s(x)$, to show that the conditions of Lemma~\ref{l:induct} are satisfied with $r_{t+1}$ and $S_{t+1}$.  We will then apply Lemma~\ref{l:induct} inductively to complete the proof of Theorem \ref{t:main-param}.

We first show that the assumption \eqref{e:ep1-def} of Lemma~\ref{l:induct} is satisfied for any $\ep_1$ satisfying inequality~\eqref{e:main1}. The first condition of Lemma~\ref{l:escape} holds by assumption on the $f_s$'s. To see that the second condition holds for the values $r_{t+1}$ and $S_{t+1}$, note that
by~\eqref{e:induct-sg-bound0} and Lemma~\ref{l:var}, when the event $G_{t}\cap H_{t}$ occurs, 
and when $\ve{X_{t+1}^i-x_{t+1}^\star}\le r_{t+1}$, 
the stochastic gradient $g^{t+1}_i$ in~\eqref{e:toymc} satisfies $\ve{g^{t+1}_i}\le S_{t+1}$.  Therefore, by Lemma~\ref{l:escape} and by inequality~\eqref{e:main1} we have $\Pj\pa{G_{t+1} | G_{t}\cap H_{t}
} \ge 1-\ep_1$. 
%
 Hence, we have that inequality \eqref{e:ep1-def} of Lemma~\ref{l:induct} is satisfied for any $\ep_1$ satisfying inequality~\eqref{e:main1}.

Next, we note that assumption \eqref{eq:eta_imax_assumption} of Lemma~\ref{l:induct} is satisfied since Inequalities~\eqref{e:main2},~\eqref{e:main3}, and~\eqref{e:main5} ensure that $\eta_0$, $i_{\max}$, and $C$ satisfy the inequalities in \eqref{eq:eta_imax_assumption}.

Therefore we have that all the conditions of Lemma~\ref{l:induct} are satisfied.  Recall we are proving~\eqref{e:indhyp1} by induction for $t=2^a$. By the above, we know we can apply Lemma~\ref{l:induct} for any $t<T$.

\paragraph{Base case of induction.} We show~\eqref{e:indhyp1} holds for $t=1$. By assumption $\ve{X^0-x_0^\star} \le \fc{C_1}{\sqrt{L_0/L}}$ so $H_0$ holds and the $t=0$ case of Lemma~\ref{l:induct} shows $\Pj(G_1)\ge 1-\ep_1$ and $\Pj(G_1\cap H_1)\ge 1-(\ep_1+\ep_2+Ae^{-kC_1}) \ge 1-(\ep_1 + 2\ep_2)$, using~\eqref{e:main4} for the last inequality.

\paragraph{\eqref{e:indhyp1} implies~\eqref{e:indhyp1.1}, \eqref{e:indhyp2}.} This follows from parts 1 and 2 of Lemma~\ref{l:induct}, as follows. Let $A_t=G_t\cap H_t$. Let $t=2^a+b$, $0<b\le 2^a$.

For~\eqref{e:indhyp1.1}, using part 1 of Lemma~\ref{l:induct} and the induction hypothesis,
\begin{align}
\Pj((G_t\cap H_{2^a})^c)&\le 
\Pj(G_t^c|A_{2^a}) + \Pj(A_{2^a}^c)\\
&\le (t-2^a)\ep_1 + [2^a \ep_1+2(a+1)\ep_2]=t\ep_1 + 2(a+1)\ep_2.
\end{align}

For~\eqref{e:indhyp2}, note that by part 2 of of Lemma~\ref{l:induct}, conditioned on $A_{2^a}$, $\ve{\cal L(X_t)-\pi_t}_{TV}\le (t-2^a)\ep_1+\ep_2$. Without the conditioning,
\begin{align}
\ve{\cal L(X_t)-\pi_t}_{TV}&\le  [(t-2^a)\ep_1+\ep_2]+\Pj(A_{2^a}^c) \\
&\le [(t-2^a)\ep_1+\ep_2]+ [2^a \ep_1+2(a+1)\ep_2]
\le 2^a\ep_1+(2a+3)\ep_2.
\end{align}

\paragraph{Induction step.} We show that if~\eqref{e:indhyp1} holds for $t$, then it holds for $2t$. We work with the complements. 
By a union bound, 
\begin{align}
\Pj(A_{2t}^c) \leq \Pj(A_{2t}^c\cap A_t) + \Pj(A_t^c)\le \Pj(A_{2t}^c|A_t) + \Pj(A_t^c).
\end{align}
The first term is bounded by Part 3 of Lemma~\ref{l:induct} and~\eqref{e:main4}, $P(A_{2t}^c|A_t)\leq t \ep_1 + \ep_2 + \ep_2$. The second term is bounded by the induction hypothesis, which says $P(A_{t}^c) \leq t\ep_1 + 2(a+1)\ep_2$. Combining these gives $P(A_{2t}^c) \leq 2t\ep_1 + 2(a+2)\ep_2$, completing the induction step.

\paragraph{Showing inequalities.}
Setting $C_1$, $\eta_0$, and $i_{\max}$ as in~\eqref{e:main-C1-online},~\eqref{e:main-eta0}, and~\eqref{e:main-imax} (with $\fR$ to be determined), we get that~\eqref{e:main2},~\eqref{e:main3}, and~\eqref{e:main4} are satisfied, as $\fR\ge \sfc{d}{L}$, $b\ge 9$ imply $\fc{{\ep_2}^2}{2L^2(\fR+\fD)^2} \le \fc{{{\ep_2}^2}}{Ld + 9L^2 (\fR+\fD)^2/b}$. 
Moreover, setting $C_\xi = \sqrt{2d+8\log \pf{2i_{\max}}{\ep_1}}$ makes $i_{\max}\exp\pa{-\fc{C_\xi^2-d}{8}}\le \fc{\ep_1}2$. 
It suffices to show that our choice of $\fR$ makes
\begin{align}
\fc{\ep_1}{2i_{\max}} 
&\ge 
  \exp\pa{-\fc{(r^2 -  
   \fc{16(C_1+\fD)^2}{t+L_0/L} 
-  i_{\max} [2\eta_t^2 (S_t^2 + L^2(t+L_0/L)^2 r_t^2) + \eta_t d])^2}{2i_{\max}( 2\eta_t S_t r_t + 2 \sqrt{\eta_t} C_\xi (r_t+\eta_t S_t + \eta_t L(t+L_0/L) r_t) + \eta_t C_\xi^2 )^2}}
\end{align}
It suffices to show 
\begin{align}
{\log\pf{2i_{\max}}{\ep_1}}
&\le 
\fc{(r_t^2 -  
   \fc{16(C_1+\fD)^2}{t+L_0/L} 
-  i_{\max} [2\eta_t^2 (S_t^2 + L^2(t+L_0/L)^2 r_t^2) + \eta_t d])^2}{2i_{\max}( 2\eta_t S_t r_t + 2 \sqrt{\eta_t} C_\xi (r_t+\eta_t S_t + \eta_t L(t+L_0/L) r_t) + \eta_t C_\xi^2 )^2}\\
\Leftarrow
r_t^2 &\ge \sqrt{2i_{\max}}
\pa{2\eta_t S_tr_t + 2\sqrt{\eta_t}C_\xi (r_t+\eta_t S_t + \eta_t L (t+L_0/L) r_t) + \eta_t C_\xi^2}
\sqrt{\log\pf{2i_{\max}}{\ep_1}}
\\
&\quad + \fc{16(C_1+\fD)^2}{t+L_0/L} + i_{\max}[2\eta_t^2 (S_t^2+L^2(t+L_0/L)^2r_t^2) + \eta_t d]
\end{align}
Substituting~\eqref{e:eta-t},~\eqref{e:r-t}, and~\eqref{e:s-t}, this is equivalent to
\begin{align}
\fc{\fR^2}{t+\fc{L_0}{L}}
&\ge \fc{\sqrt{2i_{\max}\eta_0}}{t+\fc{L_0}L}\Bigg[
\Bigg(
\fc{2\sqrt{\eta_0} 6 \sqrt t L(\fR+\fD)\fR}{\sqrt{t+\fc{L_0}L}} + 2C_\xi\pa{ \fR + \fc{\eta_06\sqrt t L(\fR+\fD)}{\sqrt{t+\fc{L_0}{L}}} + \eta_0L\fR} \\
&\quad + \sqrt{\eta_0}C_\xi^2\Bigg) \sqrt{\log \pf{2i_{\max}}{\ep_1}}\\
&\quad
+ \fc{16(C_1+\fD)^2}{t+\fc{L_0}L} 
+ \fc{i_{\max}\eta_0}{t+\fc{L_0}L}
\ba{
\fc{2\eta_0}{t+\fc{L_0}L}
\pa{36 tL^2(\fR+\fD)^2 + L^2\pa{t+\fc{L_0}L}\fR^2}+d}\Bigg]
\\
\Leftarrow
\fR^2 &\ge 
\sqrt{2i_{\max}\eta_0}
(12\sqrt{\eta_0}L(\fR+\fD)\fR
+ 2C_\xi (\fR + 6\eta_0 L(\fR+\fD) + \eta_0L\fR) \\
&\quad + \sqrt{\eta_0}C_\xi^2)\sqrt{\log \pf{2i_{\max}}{\ep_1}}\\
&\quad+16(C_1+\fD)^2 
+ i_{\max}\eta_0 
\ba{
\fc{2\eta_0}{t+\fc{L_0}L}
(36 tL^2(\fR+\fD)^2 + L^2 \pa{t+\fc{L_0}{L}} \fR^2)+d}
\end{align}
Using $\eta_0=\fc{\ep_2^2}{2L^2\fR^2}$, $i_{\max} = \ff{20(C_1+\fD)^2}{\eta_0\ep_2^2}\le \fc{40(C_1+\fD)^2}{\eta_0\ep_2^2}$, and $i_{\max}\eta_0\le \fc{40(C_1+\fD)^2}{\ep_2^2}$, the RHS is at most
\begin{align}
&
\sqrt{2i_{\max}\eta_0}
\pa{12\sqrt{\eta_0}L(\fR+\fD)\fR
+ 2C_\xi (\fR + 7\eta_0 L(\fR+\fD) ) + \sqrt{\eta_0}C_\xi^2}\sqrt{\log \pf{2i_{\max}}{\ep_1}}\\
&\quad+16(C_1+\fD)^2 
+ i_{\max}\eta_0 
\ba{
2\eta_0
(37 L^2(\fR+\fD)^2 )+d}\\
&\le 
\fc{\sqrt{80}(C_1+\fD)}{\ep_2}
\pa{6\sqrt 2\ep_2 \fR + 2C_\xi \pa{\fR+\fc{7\ep_2^2}{2L\fR}} + \fc{\ep_2C_\xi^2}{\sqrt 2L\fR}}\sqrt{\log \pf{2i_{\max}}{\ep_1}} \\
&\quad + 16(C_1+\fD)^2 + \fc{40(C_1+\fD)^2}{\ep_2^2}
(37 \ep_2^2 +d).
\end{align}
Let $Q={\log\pf{2i_{\max}}{\ep_1}}$. 
It suffices to show each of the 5 terms is at most $\fc{\fR^2}5$. Below, we use $C_\xi \le4 \sqrt{d\log\pf{2i_{\max}}{\ep_1}}$.  
\begin{align}
\label{e:r25-1}
\fc{\fR^2}5 &\ge 24\sqrt{10}(C_1+\fD) \fR \sqrt Q
&\Leftarrow  \fR &\ge 120\sqrt{10} (C_1+\fD) \sqrt{\log\pf{2i_{\max}}{\ep_1}}
\\
\label{e:r25-2}
\fc{\fR^2}5 &\ge \fc{8\sqrt 5 (C_1+\fD)C_\xi}{\ep_2} \pa{\fR+\fc{7\ep_2}{2L\fR}}\sqrt Q
&\Leftarrow \fR^2 &\ge \fc{
160\sqrt 5(C_1+\fD)}{\ep_2} \pa{\fR+\fc{7\ep_2}{2L\fR}}\sqrt{d Q}
\\
\label{e:r25-3}
\fc{\fR^2}5 & \ge\fc{2\sqrt{10}(C_1+\fD)C_\xi^2}{L\fR} \sqrt Q&
\Leftarrow \fR^3 &\ge \fc{
160\sqrt{10} (C_1+\fD)}{L} d Q^{\fc 32} 
\\
\label{e:r25-4}
\fc{\fR^2}5 &\ge 16(C_1+\fD)^2\\
\label{e:r25-5}
\fc{\fR^2}5 &\ge 40(C_1+\fD)^2\pa{40+\fc{d}{\ep_2^2}}
\end{align}
It remains to check each of these five inequalities. First, we bound $Q$. 
\begin{align}
i_{\max} &\le \fc{40L^2(\fR+\fD)^2\pa{C_1+\fD}^2}{{\ep_2}^4},\\
\fc{2i_{\max}}{\ep_1}
&\le 
\fc{80L^2(\fR+\fD)^2\pa{C_1+\fD}^2}{{\ep_2}^4 \ep_1}\\
&\le 
\fc{100L^2 \fR^2 \pa{C_1+\fD}^2}{{\ep_2}^4 \ep_1}\\
&\le 
\fc{10^{10} L^2 (C_1+\fD)^4 d}{\ep_2^6\ep_1} \log^2\pa{\max\bc{L,C_1+\fD, \rc{\ep_1}}}
\\
\log \pf{2i_{\max}}{\ep_1} 
&\le 
24 + 16\log\pa{ \max\bc{L,C_1+\fD,\rc{\ep_1}}}\\
&\le 40 \log\pa{ \max\bc{L,C_1+\fD,\rc{\ep_1}}}
\end{align}
It remains to check~\eqref{e:r25-1}--\eqref{e:r25-5}. We check~\eqref{e:r25-1},~\eqref{e:r25-2}, and~\eqref{e:r25-3}:
\begin{align}
120\sqrt{10}(C_1+\fD) \sqrt Q
&\le 120\sqrt{10} (C_1+\fD) 
\sqrt{40 \log\pa{ \max\bc{L,C_1+\fD,\rc{\ep_1}}}}
\le \fR
\end{align}
Using $\fR\ge \sfc{7\ep_2}{2L}\implies \fc{7\ep_2}{2L\fR}\le \fR$,
\begin{align}
\fc{160\sqrt {5}(C_1+\fD)}{\ep_2} \pa{\fR+\fc{7\ep_2}{2L\fR}}\sqrt d Q
&\le \fc{320\sqrt{10}(C_1+\fD) \sqrt d\fR}{\ep_2}{40 \log\pa{ \max\bc{L,C_1+\fD,\rc{\ep_1}}}}\le \fR^2\\
 \fc{160\sqrt {10}(C_1+\fD)}{L} \pa{\fR+\fc{7\ep_2}{2L\fR}}\sqrt{d}Q^{\fc 32}
 &\le \fc{80\sqrt{10}(C_1+\fD)d}{L}\pa{40 \log\pa{ \max\bc{L,C_1+\fD,\rc{\ep_1}}}}^{\fc 32} \le \fR^3.
\end{align}
The last two inequalities~\eqref{e:r25-4},~\eqref{e:r25-5} are immediate from the definition of $\fR$.
\end{proof}

%% file: logistic.tex
\section{Proof for logistic regression application}

\label{s:bayes}
\subsection{Theorem for general posterior sampling, and application to logistic regression}

We show that under some general conditions---roughly, that we see data in all directions---the posterior distribution concentrates. We specialize to logistic regression and show that the posterior for logistic regression concentrates under reasonable assumptions. 

The proof shares elements with the proof of the Bernstein-von Mises theorem (see e.g. \cite{nickl2012statistical}), which says that under some weak smoothness and integrability assumptions, the posterior distribution after seeing iid data (asymptotically) approaches a normal distribution. 
However, we only need to prove a weaker result---not that the posterior distribution is close to normal, but 
just $\al T$-strongly log concave in a neighborhood of the MLE, for some $\al>0$; hence, we get good, nonasymptotic bounds. 
This is true under more general assumptions; in particular, the data do not have have to be iid, as long as we observe data ``in all directions.''

\begin{thm}[\bf Validity of the assumptions for posterior sampling]\label{thm:general-conc}
Suppose that %
$\ve{\te_0}\leq B$, $x_t\sim P_x(\cdot |x_{1:t-1},\te_0)$. Let $f_t$, $t\ge 1$ be such that $P_x(x_t|x_{1:t-1}, \te)\propto  e^{-f_t(\te)}$ and let $\pi_t(\te)$ be the posterior distribution, $\pi_t(\te)\propto e^{-\sumz kt f_t(\te)}$. Suppose there is $M,L,r,\smin,T_{\min}>0$ and $\al,\beta\geq 0$ such that the following conditions hold:
\begin{enumerate}
\item For each $t$, $1\le t\le T$, $f_t(\te)$ is twice continuously differentiable and convex.
\item \label{i:bd-grad}
(Gradients have bounded variation) 
For each $t$, given $x_{1:t-1}$, 
\begin{align}\ve{\nb f_t(\te) - \E [\nb f_t(\te) | x_{1:t-1}]}\leq M.\end{align}
\item \label{i:smooth}
(Smoothness) Each $f_t$ is $L$-smooth, for $1\le t\le T$.
\item \label{i:sc-nbhd}
(Strong convexity in neighborhood) 
Let 
\begin{align}
\wh I_T(\te) :&= \rc T\sumo tT \nb^2 f_t(\te).
\end{align}
Then for $T\geq T_{\min}$, with probability $\geq 1-\fc{\ep}2$, %
\begin{align}\label{e:IT}
\forall \te&\in \mathrm{B}(\te_0,r),&
\wh I_T(\te) &\succeq \si_{\min} I_d.
\end{align}

\item \label{i:prior}
 $f_0(\te)$ is $\al$-strongly convex and $\beta$-smooth, and has minimum at $\te=0$.
\end{enumerate}
Let $\te_T^\star$ be the minimum of $\sumz tT f_t(\te)$, i.e., the mode for $\te$ after observing $x_{1:T}$. 
Letting 
\begin{align*}
C&=
\max\bc{1,M\sqrt{2d\log \pf{2d}{\ep}}, \fc{4d}{\smin}},
\end{align*}
and $c=\fc{\al}{\smin}$, 
if $T\geq T_{\min}$ is such that $\slb + \fc{C}{\sqrt{T+c}}<r$, then with probability $1-\ep$, the following hold:
\begin{enumerate}
\item
$\ve{\te_T^\star-\te_0}\leq \slb$. 
\item
For $C'\ge 0$, 
$\Pj_{\te\sim \pi_T}\pa{\ve{\te-\te_T^\star}\ge \fc{C'}{\sqrt{T+c}}}\le 
\fc{K_1}{\smin C\sqrt{T+c}} 
\pf{(LT+\beta)e}{d}^{\fc d2} 
e^{\rc 2\smin C^2 - \fc{\smin CC'}{2}}
$ for some constant $K_1$.
\end{enumerate}
\end{thm}

\noindent
The strong convexity condition is analogous to a small-ball inequality \cite{koltchinskii2015bounding,mendelson2014learning} for the sample Fisher information matrix in a neighborhood of the true parameter value. In the iid case we have concentration (which is necessary for a central limit theorem to hold, as in the Bernstein-von Mises Theorem); in the non-iid case we do not necessarily have concentration, but the small-ball inequality can still hold.

We show that under reasonable conditions on the data-generating distribution, logistic regression satisfies the above conditions.
Let $\phi(x) = \rc{1+e^{-x}}$ be the logistic function. Note that $\phi(-x) = 1-\phi(x)$.

Applying Theorem~\ref{thm:general-conc} to the setting of logistic regression, we will obtain the following. 
\begin{lem}
\label{thm:logistic-conc}
In the setting of Problem~\ref{p:log-reg} (logistic regression), 
suppose that $\ve{\te_0}\leq \fB$,
$u_t\sim P_u$ are iid, where $P_u$ is a distribution that satisfies the following: for $u\sim P_u$, 
\begin{enumerate}
\item
(Bounded) $\ve{u}_2\leq M$ with probability 1.
\item
(Minimal eigenvalue of Fisher information matrix)
\begin{align}I(\te_0):&=\int_{\mathbb{R}^d} \phi(u^\top \te_0)\phi(-u^\top \te_0) u u^\top \,dP_u\succeq \si I_d,\end{align} for $\si>0$.
\end{enumerate}
Let
\begin{align}
C&=\max\bc{1, 2M \sqrt{2d\log\pf{2d}{\ep}}, \fc{4ed}{\si}}.
\end{align}
 %
Then for $t>  \max\bc{\fc{M^4 \log \pf{2d}{\ep}}{8\si^2} ,4M^2\pa{\fc{2eC}{\si}+1}^2,\fc{4eM\fB\al}{\si}}$, we have
\begin{enumerate}
\item
$\nabla f_k(\te)$ is $\fc{M^2}{4}$-Lipschitz for all $k\in \mathbb{N}$.
\item For any $C'\ge 0$, and $c=\fc{2e\al}{\si}$,  \begin{align}\Pj_{\te\sim \pi_t}\pa{\ve{\te-\te_t^\star}\ge\fc{C'}{\sqrt{T+c}}}\le
\fc{K_1}{\si C\sqrt{T+c}} 
\pf{\pa{\fc{M^2}4T+\alpha}e}{d}^{\fc d2} 
e^{\rc{4e}\si C^2 - \fc{\si CC'}{4e}}
\end{align} for some constant $K_1$.
\item
With probability $1-\ep$, $\ve{\te_t^\star-\te_0}\le \fc{C\sqrt t+\al \mathfrak{B}}{\si t/2e + \al}$.
\end{enumerate}
\end{lem}

\begin{rem}\label{rem:cov}
We explain the condition $I(\te_0)= \int_{\mathbb{R}^d} \phi(u^\top \te_0)\phi(-u^\top \te_0) u u^\top \,dP_u\succeq \sigma I_d$. Note that $\phi(x)\phi(-x)$ can be bounded away from 0 in a neighborhood of $x=0$, and then decays to 0 exponentially in $x$. Thus, $I(\te_0)$ is essentially the second moment, where we ignore vectors that are too large in the direction of $\pm \te_0$. 

More precisely, we have the following implication:
\begin{align}
\E_u [uu^\top \one_{\phi(u^\top \te_0)\leq C_1}] \succeq \si I_d \implies
\int_{\mathbb{R}^d}\phi(u^\top \te_0)\phi(-u^\top \te_0)uu^\top \,dP_u\succeq \rc{\phi(C_1)(1-\phi(C_1))} \si I_d.
\end{align}
Theorem~\ref{thm:main-logistic-conc} is stated with $C_1=2$.
\end{rem}

\subsection{Proof of Theorem~\ref{thm:general-conc}}
\begin{proof}[Proof of Theorem~\ref{thm:general-conc}]
Let $\cal E$ be the event that~\eqref{e:IT} holds.

\step{1} We bound $\ve{\te_{T}^\star - \te_0}$ with high probability.

We show that with high probability $\sumz tT \nb f_t(\te_0)$ is close to 0. Since $\sumz tT \nb f_t(\te_{T}^\star)=0$, the gradient at $\te_0$ and $\te_{T}^\star$ are close. Then by strong convexity, we conclude $\te_0$ and $\te_{T}^\star$ are close. 

First note that $\E[f_t(\te)|x_{1:t-1}] = \int_{\mathbb{R}^d} -\log P_x(x_t|x_{1:t-1},\te) \,dP_x(\cdot |x_{1:t-1},\te_0)$
is a KL divergence minus the entropy for $P_x(\cdot |x_{1:t-1},\te_0)$, and hence is minimized at $\te=\te_0$. Hence $\rc T \sumo tT \E[\nb f_t(\te_0)|x_{1:t-1}]=0$.  
Thus by Lemma~\ref{lem:azuma-d} applied to 
\begin{align}
\sumo tT \nb f_t(\te_0) &= 
\sumo t T \ba{\nb f_t(\te_0) - \E[\nb f_t(\te_0)|x_{1:t-1}] },
\end{align}
we have by Chernoff's inequality that 
\begin{align}
\Pj\pa{\ve{\sumo t T \nb f_t(\te_0)}\geq \fc{C}{\sqrt T}} &\leq 
2de^{-\fc{C^2}{2M^2d}}\leq \fc{\ep}2
\end{align}
when $\fc{C^2}{2M^2d} \geq \log \pf{4d}{\ep}$, which happens when $C\geq M\sqrt{2d \log \pf{4d}{\ep}}$.

Let $\cal A$ be the event that 
$\ve{\rc T\sumo t T \nb f_t(\te_0)}< \fc{C}{\sqrt T}$. 
Then under $\cal A$,
\begin{align}
\ve{\rc T \sumz tT \nb f_t(\te_0)} &> -\fc{C}{\sqrt T}-
\rc T \beta\ve{\te_0}
\geq -\fc{C}{\sqrt T}- \fc{\beta B}{T}.
\end{align}
Let $w=\nv{\te_T^\star-\te_0}$. Under the event $\cal E$,
\begin{align}
\rc T\sumz tT \nb f_t(\te_0+sw)^\top w &\geq -\fc{C}{\sqrt T} -\fc{\beta B}T + \pa{\si_{\min} + \fc{\al}{T}}\min\{s,r\}.
\end{align}
Hence, if $s,r>\slb$, then $\sumz tT \nb f_t(\te_0)\ne 0$. Considering $s=\ve{\te_T^\star-\te_0}$, this means that
\begin{align}
\ve{\te_T^\star - \te_0} &\leq \slb. 
\end{align}

\step{2} For $c=\fc{\al}{\smin}$, we bound $\Pj_{\te\sim \pi_T}(\ve{\te-\te_T^\star}\ge \fc{C'}{\sqrt{T+c}})$. 

Under $\cal E$, $\rc T\sumo tT f_t(\te)$ is $\si_{\min}$-strongly convex for $\te\in \mathrm{B}\pa{\theta_T^\star, \fc{C}{\sqrt{T+c}}}\sub \mathrm{B}(\te_0, r)$, and $f_0(\te)$ is $\al$-strongly convex. 

Let $r'= r-\slb$. Under $\cal A$, $\mathrm{B}(\te_T^\star, r') \sub \mathrm{B}(\te_0,r)$. 
Thus under $\cal E\cap \cal A$, letting $w(\te) := \nv{\te-\te_T^\star}$,
\begin{align}
\forall \te&\in \mathrm{B}(\te_T^\star, r') \sub \mathrm{B}(\te_0,r), &
\sumz tT  \nb f_t(\te)^\top w(\te)  \geq \pa{T\si_{\min}+\al} \ve{\te-\te_T^\star}.
\end{align}
Suppose $T$ is such that $\fc{C}{\sqrt{T+c}}<r'$, i.e., $\slb + \fc{C}{\sqrt{T+c}}<r$.
By shifting, we may assume that $\sumz tT f_t(\te_T^\star)=0$. 
Because $f_t(\te)$ is $L$-smooth for $1\leq t\leq T$ and $\beta$-smooth for $t=0$, 
\begin{align}
\sumz tT f_t(\te) &\leq \fc{LT+\beta}{2}\ve{\te-\te_T^\star}^2.
\end{align}
Then for all $\te\in \mathrm{B}\pa{\te_T^\star, \fc{C}{\sqrt {T+c}}}^c$,
\begin{align}
\sumz tT f_t(\te) &\geq 
\sumz tT f_t\pa{\te_T^\star + \fc{C}{\sqrt{T+c}}w(\te)}
+\sumz tT \ba{f_t(\te) - f_t\pa{\te_T^\star + \fc{C}{\sqrt{T+c}}w(\te)}}\\
&\ge \rc 2(T\smin+ \al)\fc{C^2}{T+c}+(T\si_{\min}+\al)\fc{C}{\sqrt{T+c}}\pa{\ve{\te-\te_T^\star}-\fc{C}{\sqrt{T+c}}}\\
&\ge \rc 2 \smin C^2 + \smin C\sqrt{T+c} \pa{\ve{\te-\te_T^\star}-\fc{C}{\sqrt{T+c}}}.
\end{align}
Thus for any $C'\ge 0$,
\begin{align}
\label{e:w2-lb}
\int_{\mathbb{R}^d} e^{-\sumz tT f_t(\te)}\,d\te
&\geq 
\int_{\mathbb{R}^d} e^{-\fc{LT+\beta}2\ve{\te-\te_T^\star}^2}\,d\te = \pf{2\pi}{LT+\beta}^{\fc d2},\\
\int_{\mathrm{B}\pa{\te_T^\star, \fc{C'}{\sqrt{T+c}}}^c} 
e^{-\sumz tT f_t(\te)}\,d\te
 &\leq \int_{\mathrm{B}\pa{\te_T^\star, \fc{C'}{\sqrt{T+c}}}^c} e^{-\rc 2\smin C^2} e^{-{\smin C\sqrt{T+c}}\pa{\ve{\te-\te_T^\star}-\fc{C}{\sqrt{T+c}}}}\,d\te\\
 & = \int_{\fc{C'}{\sqrt{T+c}}}^\iy \Vol_{d-1}(\bS^{d-1}) \ga^{d-1} 
 e^{\rc2\smin C^2}
 e^{-\smin C\sqrt{T+c}\ga}\,d\ga\\
 & = \int_{\fc{C'}{\sqrt{T+c}}}^\iy \Vol_{d-1}(\bS^{d-1})  
 e^{\rc2\smin C^2}
 e^{-(\smin C\sqrt{T+c}\ga-(d-1)\log \ga)}\,d\ga.
 \label{e:w2-ub1}
\end{align}
Now, when $C\geq \max\{\fc{2(d-1)}{\smin},1\}$, we have that
\begin{align}
\smin C\sqrt{T+c}\ga- (d-1)\log \ga
&\ge \smin C\sqrt{T+c}\ga- (d-1)\ga\\
&\ge \smin C \sqrt{T+c}\ga - \fc{\smin C \sqrt{T+c}\ga}2\\
&=\fc{\smin C\sqrt{T+c}\ga}2.
\end{align}
Then by Stirling's formula, for some $K_1$,
\begin{align}
\eqref{e:w2-ub1}
&\leq 
 \Vol_{d-1}(\bS^{d-1}) 
 e^{\rc 2\smin C^2}
 \int_{\fc{C'}{\sqrt{T+c}}}^\iy
 e^{-\fc{\smin C\sqrt{T+c} \ga}{2}}\,d\ga\\
 &\leq 
 \fc{2\pi^{\fc d2}}{\Ga\pf d2} 
  e^{\rc 2\smin C^2}
\fc{2}{\smin C\sqrt{T+c}}e^{-\fc{\smin CC'}2}\\
&\le  \fc{K_1}{\smin C\sqrt{T+c}} \pf{2\pi e}{d}^{\fc d2}
e^{\rc 2\smin C^2 - \fc{\smin CC'}{2}}.
\end{align}
We bound $\Pj_{\te\sim \pi_T}\pa{\ve{\te-\te_T^\star}\ge \fc{C'}{\sqrt{T+c}}}$. By~\eqref{e:w2-lb} and~\eqref{e:w2-ub1}, 
\begin{align}
\Pj_{\te \sim \pi_T} \pa{\ve{\te-\te_T^\star}\ge \fc{C'}{\sqrt{T+c}}}
&=
\fc{\int_{\te\in B\pa{\te_T^\star,\fc{C'}{\sqrt{T+c}}}^c}e^{-\sumz tT f_t(\te)}\,d\te}{\int_{\R^d} e^{-\sumz tT f_t(\te)}\,d\te}\\
&\le \fc{K_1}{\smin C \sqrt{T+c}} \pf{LT+\beta}{2\pi}^{\fc d2} 
\pf{2\pi e}{d}^{\fc d2} e^{\rc 2\smin C^2 - \fc{\smin CC'}{2}}\\
&=\fc{K_1}{\smin C\sqrt{T+c}} 
\pf{(LT+\beta)e}{d}^{\fc d2} 
e^{\rc 2\smin C^2 - \fc{\smin CC'}{2}},
\end{align}
as needed.
The requirements on $C$ are $C\geq \max\bc{1, M\sqrt{2d\log \pf{4d}{\ep}},\fc{2d}{\smin}}$, 
so the theorem follows.
\end{proof}
\subsection{Online logistic regression: Proof of Lemma~\ref{thm:logistic-conc} and Theorem~\ref{thm:main-logistic-conc}}
To prove Lemma~\ref{thm:logistic-conc}, we will apply Theorem~\ref{thm:general-conc}. To do this, we need to verify the conditions in  Theorem~\ref{thm:general-conc}.

\begin{lem}\label{lem:logistic-conc}
Under the assumptions of Lemma~\ref{thm:logistic-conc}, 
\begin{enumerate}
\item (Gradients have bounded variation) For all $t$, 
$\ve{\nb f_t(\te)}\le M$ and \\
$\ve{\nb f_t(\te) - \E \nb f_t(\te)}\leq 2M$.
\item (Smoothness) For all $t$, $f_t$ is $\rc 4M^2$-smooth.
\item (Strong convexity in neighborhood) for $T\geq \fc{M^4\log \pf{d}\ep}{8\si^2}$, 
\begin{align}
\Pj\pa{
\forall \te\in \mathrm{B}\pa{\te_0, \rc M}, \sumo tT \nb^2 f_t(\te) \succeq \fc{\si}{2e}TI_d
}
\geq 1-\ep.
\end{align} 
\end{enumerate}
\end{lem}

\begin{proof} 
First, we calculate the Hessian of the negative log-likelihood.

If $f_t(\te) = -\log \phi(yu^\top \te)$, then
\begin{align}
\nb f_t(\te) &= \fc{-y \phi(yu^\top \te) \phi(-yu^\top \te)}{\phi(yu^\top \te)} u = -y\phi(-yu^\top \te)u,\\
\nb^2 f_t(\te) &= \phi(-y u^\top \te) \phi(yu^\top \te) uu^\top.
\end{align}
Note that $\ve{\nb f_t(\te)}\leq \ve{u}\leq M$, so the first point follows.

To obtain the expected values, note that $y=1$ with probability $\phi(u^\top \te_0)$, and $y=-1$ with probability $1-\phi(u^\top \te_0)$, so that 
\begin{align}
\E [\nb^2 f_t(\te)] &= \E_{(u,y)} [\phi(-y u^\top \te) \phi(yu^\top \te) uu^\top]\\
&= \E_u [\phi(u^\top \te_0)\phi(-y u^\top \te) \phi(yu^\top \te) uu^\top + (1-\phi(u^\top \te_0)) 
\phi(-y u^\top \te) \phi(yu^\top \te) uu^\top]\\
&= \E_u[ \phi(u^\top \te)(1-\phi(u^\top \te)) uu^\top].
\end{align}
Suppose that $\E_u[ \phi(u^\top \te)(1-\phi(u^\top \te)) uu^\top]\succeq \si I$. 

Next, we show that $\sumo tT \nb^2 f_t(\te_0)$ is lower-bounded with high probability. 

Note that $\ve{\nb^2 f_t(\te_0)} = \ve{\phi(-y u^\top \te_0) \phi(yu^\top \te_0) uu^\top}_2\leq \rc 4 M^2$. (So the second point follows.) By the Matrix Chernoff bound, 
\begin{align}
\Pj\pa{
\sumo tT \nb f_t^2(\te_0) \not\succeq \fc{\si}2TI_d
} &\leq d e^{-\fc{2\cdot 4^2}{M^4}T \pf{\si}{2}^2} 
= de^{-\fc{8\si^2T}{M^4}}\leq \ep,
\end{align}
when $T\geq \fc{M^4 \log \pf{d}{\ep}}{8\si^2}$. 

Finally, we show that if the minimum eigenvalue of this matrix is bounded away from 0 at $\te_0$, then it is also bounded away from 0 in a neighborhood.
To see this, note
\begin{align}
\fc{\phi(x+c)(1-\phi(x+c))}{\phi(x)(1-\phi(x))}
& = \fc{e^{x+c}}{(1+e^{x+c})^2}
\fc{(1+e^x)^2}{e^x}
 \geq \fc{e^c}{e^{2c}}=e^{-c}.
 \label{e:logistic-ratio}
\end{align}
Therefore, if $\sumo tT \nb^2 f_t(\te_0)\succeq \si' I_d$, then for $\ve{\te-\te_0}_2\leq \rc{M}$, $|u^\top \te - u^\top \te_0|<1$ so by~\eqref{e:logistic-ratio},
\begin{align}
\sumo tT \nb^2 f_t(\te) &= \sumo tT
\phi(u_t^\top \te)(1-\phi(u_t^\top \te)) u_tu_t^\top
\\
&\succeq 
\sumo tT 
e^{-1} \phi(u_t^\top \te_0)(1-\phi(u_t^\top \te_0)) u_tu_t^\top
 \succeq \fc{\si'}eI_d.
\label{e:fi-lb}
\end{align}
Therefore, 
\begin{align}
\Pj\pa{
\forall \te\in \mathrm{B}\pa{\te_0, \rc M}, \sumo tT \nb^2 f_t(\te) \not\succeq \fc{\si}{2e}TI_d
}
&\leq \Pj\pa{
\sumo tT \nb f_t^2(\te_0) \not\succeq \fc{\si}2TI_d
}\leq \ep. 
\end{align}
\end{proof}

\begin{proof}[Proof of Lemma~\ref{thm:logistic-conc}]
Part 1 was already shown in Lemma~\ref{lem:logistic-conc}.

Lemma~\ref{lem:logistic-conc} shows that the conditions of Theorem~\ref{thm:general-conc} are satisfied with $M\mapsfrom 2M$, $L=\fc{M^2}4$, $r=\rc M$, $\smin = \fc{\si}{2e}$, $T_{\min}=\fc{M^4\log\pf{2d}{\ep}}{8\si^2}$. Also, $\al=\beta$. We further need to check that 
the condition on $t$ implies that $\fc{C\sqrt t+\beta \fB}{\smin t + \al}+\fc{C}{\sqrt t}<\rc M$. We have, noting $\smin\le L$ (the strong convexity is at most the smoothness),
\begin{align}
\fc{C\sqrt t+\beta \fB}{\smin t + \al}+\fc{C}{\sqrt t}
&\le \pa{\fc{C}{\smin}+1}
\rc{\sqrt{t+\fc{\al}L}}
+ 
\fc{\beta \fB}{\smin \pa{t+\fc{\al}{\smin}}},
\end{align}
so it suffices to have each entry be $<\rc{2M}$, and this holds when $t>4M^2\pa{\fc{C}{\smin}+1}^2 = 4M^2\pa{\fc{2eC}{\si}+1}^2$ and $t>\fc{2M\fB \beta}{\smin} = \fc{4eM\fB\al}{\si}$. 

Parts 2 and 3 then follow immediately. 
\end{proof}

\begin{proof}[Proof of Theorem \ref{thm:main-logistic-conc}]
Redefine $\si$ such that $I(\te_0)\succeq \si I_d$ holds. (By Remark~\ref{rem:cov}, this $\si$ is a constant factor times the $\si$ in Theorem~\ref{thm:main-logistic-conc})
Theorem~\ref{thm:main-logistic-conc} follows from Theorem~\ref{thm:main_online} once we show that Assumptions~\ref{a:smooth},~\ref{a:conc}, and~\ref{a:mle} are satisfied. Assumption~\ref{a:smooth} is satisfied with $L_0=\al$ and $L=\fc{M^2}4$. The rest will follow from Lemma~\ref{thm:logistic-conc} except that we need bounds to cover the case $t\le T_{\min}:=  \max\bc{\fc{M^4 \log \pf{2d}{\ep}}{8\si^2} ,\fc{16e^2M^2C^2}{\si^2},\fc{4eM\fB\al}{\si}}$ as well.\\

\noindent
\textbf{Showing that Assumption \ref{a:conc} holds.} Note $L\ge \si$ so $\fc{C'}{\sqrt{T+\fc{\al}{L}}} \ge \fc{C'}{\sqrt{T+\fc{2e\al}{\si}}}$. 
For $t>T_{\min}$, part 2 of Lemma~\ref{thm:logistic-conc} shows Assumption \ref{a:conc} is satisfied with $c=\fc{\al}{L}$ (where $L=\fc{M^2}4$), $A_1 =
\fc{K_1}{\si C} 
\pf{\pa{\fc{M^2}4T+\alpha}e}{d}^{\fc d2} 
e^{\rc{4e}\si C^2}$ and $k_1 = \fc{\si C}{4e}$.

For $t\le T_{\min}$, we use \iftoggle{thesis}{Lemma~\ref{lem:conc-mode}}{Lemma F.10 of~\cite{ge2018simulated}}, which says that if $p(x)\propto e^{-f(x)}$ in $\R^d$ and $f$ is $\ka$-strongly convex and $K$-smooth, and $x^\star=\amin_x f(x)$, then
\begin{align}
\Pj_{x\sim p}\pa{
	\ve{x-x^\star}^2 \ge \rc{\ka} 
	\pa{\sqrt{d} + \sqrt{2t + d\log\pf{K}{\ka}}
}^2}\le e^{-t}.
\end{align}
In our case, $\sumz st f_s(x)$ is $\al$-strongly convex and $\al+T_{\min}L$-smooth, so 
\begin{align}
\Pj_{x\sim p}\pa{\ve{x-x^\star}\ge \ga}
&\le \exp\ba{
	-\ba{
		\fc{(\ga\sqrt \ka - \sqrt d)^2 - d\log \pf{K}{\ka}}{2}
	}
}\\
&=e^{\fc d2\pa{-1+\log\pf{K}{\ka}}}e^{\ga\sqrt{\ka d} - \fc{\ga^2\ka}2}\\
&\le e^{\fc d2\pa{-1+\log\pf{K}{\ka}}-\pa{\ga-2\sfc d\ka}\sqrt{\ka d}}.
\end{align}
Thus for $t\le T_{\min}$,  
\begin{align}
\Pj_{\te\sim \pi_t}(\ve{\te-\te_t^\star}\ge \ga)&\le A_2e^{-k_2\ga}\\
\text{with }A_2&=e^{\fc d2\pa{-1+\log\pf{K}{\ka}}} = e^{\fc d2\pa{-1+\log \pf{T_{\min}L+\al}{\al}}}\\
k_2 &=\fc{\sqrt{\ka d}}{\sqrt{T_{\min}+\fc{\al}L}}= \fc{\sqrt{\al d}}{\sqrt{T_{\min}+\fc{\al}L}}.
\end{align}
Take $A=\max\{A_1,A_2\}$ and $k=\min\{k_1,k_2\}$ and note that $\log(A)$, $k^{-1}$ are polynomial in all parameters and $\log(T)$. \\

\noindent \textbf{Showing that Assumption \ref{a:mle} holds.} 
For $t>T_{\min}$, part 3 of Lemma~\ref{thm:logistic-conc} shows that with probability at least $1-\ep$, (using $L\ge \si$)
\begin{equs} \label{eq:logistic4}
  \ve{\te_t^\star-\te_0}\leq \frac{C \sqrt{t} + \alpha \fB}{\si t/2e +\alpha}
  \le \pa{\fc{C}{\si/2e}+\fc{\alpha \fB}{\si/2e\cdot\sqrt{t+\fc{2e\al}{\si}}}}\rc{\sqrt{t+\fc{\al}L}}.
  \end{equs} 
Now consider $t\le T_{\min}$. Since $F_t$ is strongly convex, the minimizer $\theta_t^\star$ of $F_t$ is the unique point where $\nabla F_t(\theta_t^\star) = 0$. Moreover, $\|\sum_{k=1}^{t} \nabla f_k(\te) \| \leq T_{\mathrm{min}} M$ for $t\leq T_{\mathrm{min}}$.  Therefore, since $f_0$ is $\al$-strongly convex, we have that $\ve{\nb F_t(\te)} = \ve{\nb f_0(\te) + \sumo kt \nb f_k(\te)}>0$ for all $\ve{\te}>T_{\min}M\al^{-1}$. Therefore, we must have that  $\|\theta_t^\star\| \leq T_{\mathrm{min}} M \alpha^{-1}$ for all $t\leq T_{\mathrm{min}}$, and hence that
 \begin{equs} \label{eq:logistic3}
 \|\theta_t^\star - \te_0\| \leq T_{\mathrm{min}} M \alpha^{-1} + \mathfrak{B} \qquad \forall t\leq T_{\mathrm{min}}.
 \end{equs}
 Set  $\mathfrak{D} = 2\max\bc{(T_{\mathrm{min}} M\alpha^{-1} + \mathfrak{B}) \sqrt{T_{\min}+\fc{\al}{L}}, \, \,  \fc{C}{\si /2e}+\fc{\sqrt{\al} \fB}{\sqrt{\si/2e} }}$. Then Equations \eqref{eq:logistic4} and \eqref{eq:logistic3} and the triangle inequality would imply that if $t<\tau$, then $\ve{\te_t^\star-\te_\tau^\star}\le \fc{\fD}{\sqrt{t+\fc{\al}{L}}}$. 
To get Assumption \ref{a:mle} to hold with probability at least $1-\ep$ for all $t,\tau<T$, substitute $\ep\mapsfrom \fc{\ep}{T}$. $\fD$ is polynomial in all parameters and $\log(T)$.
\end{proof}

%% file: offline-results.tex
\section{Results in the offline setting} \label{sec:offline}
\ift{\nomenclature[3piTbeta]{$\pi_T^\be$}{Distribution at inverse temperature $\beta$, $\pi_T(x) \propto e^{-\sumo tT f_t(x)}$ (Chapter 3)}}
In the offline setting, we have access to all the $f_t$'s from the start.  
Our goal is simply to generate a sample from the single target distribution $\pi_T(x) \propto e^{-\sumo tT f_t(x)}$ with TV error $\epsilon$.  
Since we do not assume that the $f_t$'s are given in any particular order, we replace Assumption \ref{a:wass} which depends on the order in which the functions are given, with an assumption (Assumption \ref{a:wass2}) on the target $\sumo tT f_t(x)$ which does not depend on the $f_t$'s ordering.  
In place of working with the sequence of distributions $\pi_1, \pi_2 \ldots$ which depend on the $f_t$'s ordering,  we introduce an inverse temperature parameter $\beta>0$ and consider the distributions $\pi_T^\beta(x) \propto e^{-\beta \sumo tT f_t(x)}$.  
In place of Assumption \ref{a:wass}, we assume:
%
\begin{assumption}[\bf Bounded second moment with exponential concentration (with constants $A, k >0$)]  \label{Assumption:Wass2} \label{a:wass2}
For all $\frac{1}{T} \leq \beta \leq 1$ and 
 all $s\ge 0$, 
$\Pj_{X\sim \pi_T^\beta}(\ve{X-x^\star}\ge \fc{s}{\sqrt{\beta T}})\le Ae^{-ks}$.
\end{assumption}
\noindent
Assumption \ref{a:wass2} says the distributions $\pi_T^\beta$ become more concentrated as $\beta$ increases from $\nicefrac{1}{T}$ to $1$.  
By sampling from a sequence of distributions $\pi_T^\beta$ where we gradually increase $\beta$ from  $\nicefrac{1}{T}$ to $1$ at each epoch, our offline algorithm (Algorithm \ref{alg:off-VRSGLD}\ifarxiv{}{ in the supplementary material}) is able to approach the target distribution $\pi_T = \pi_T^1$ when starting from a cold start that is far from a sublevel set containing most of the probability measure of $\pi_T$, without requiring strong convexity.   
Moreover, since scaling by $\beta$ does not change the location of the minimizer $x^\star$ of $\beta \sumo tT f_t(x)$, we can drop Assumption \ref{a:mle}. 

%
\begin{theorem}[\textbf{Offline variance-reduced SGLD}] \label{thm:main_offline}
Suppose that $f_1,\ldots, f_T$ satisfy Assumptions \ref{Assumption:LipschitzG} and \ref{Assumption:Wass2}.  Then there exist $b$, $\eta$,  and $i_{\mathrm{max}}$ which are polynomial in  $d, L, C,  \epsilon^{-1}$ and poly-logarithmic in $T$, such that Algorithm \ref{alg:off-VRSGLD}
generates a sample $X^T$
such that
$\|\mathcal{L}(
X^T
) - \pi_T\|_{\mathrm{TV}} \leq \epsilon.$  Moreover, the total number of gradient evaluations  is $\mathrm{polylog}(T) \times \mathrm{poly}(d,L, C, \mathfrak{D}, \epsilon^{-1}) + \tilde{O}(T)$. 
\end{theorem}
%
\noindent
See Theorem~\ref{t:main-param-offline} for precise dependencies. The theorem could also be stated with a $f_0$, but we omitted it for simplicity.
As in the online setting, we do not assume strong convexity.
Further, our additive dependence on $T$ in Theorem \ref{thm:main_offline} is tight up to log factors, since the number of gradient evaluations needed to sample from a distribution satisfying Assumptions \ref{Assumption:LipschitzG}-\ref{Assumption:MLE} is at least $\Omega(T)$ due to information theoretic requirements 
(we show this informally in \ifarxiv{Appendix~\ref{sec:Hardness}}{supplementary Appendix \ref{sec:Hardness}}).

Compared to previous work in this setting, our results are the first to obtain an additive dependence on $T$ and polynomial dependence on the other parameters without assuming strong convexity.  
While the results of \cite{chatterji2018theory} for SAGA-LD and CV-LD have additive dependence on $T$, their results require the functions $f_1,\ldots, f_T$ to be strongly convex. 
Since the basic Dikin walk and basic Langevin algorithms compute all $T$ functions or all $T$ gradients every time the Markov chain takes a step, and the number of steps in their Markov chain depends polynomially on the other parameters such as $d$ and $L$, the number of gradient (or function) evaluations required by these algorithms is \textit{multiplicative} in $T$.  
  Even though the basic SGLD  algorithm computes a mini-batch of the gradients at each step, roughly speaking the batch size at \textit{each step} of the chain should be $\Omega_T(T)$ for the stochastic gradient to have the required variance, implying that basic SGLD also has multiplicative dependence on $T$.

%% file: offline-overview.tex
\section{Overview of offline result}

\subsection{Overview of offline algorithm}
Similarly to the online Algorithm~\ref{alg:VRSGLD}, our offline Algorithm~\ref{alg:off-VRSGLD} also calls the variance-reduced SGLD Algorithm~\ref{alg:SAGA} multiple times.
In the offline setting, all functions $f_1,\ldots, f_T$ are given from the start, so there is no need to run Algorithm~\ref{alg:SAGA} on subsets of the functions. Instead, we run SAGA-LD on $\beta f_1,\ldots, \beta f_T$, where the \emph{inverse temperature} $\beta$ is doubled at each epoch, from roughly $\beta = \rc{T}$ to $\beta=1$. There are logarithmically many epochs, each taking $i_{\max} = \wt{O}_T(1)$ Markov chain steps.

Note that we cannot just run SAGA-LD on $f_1,\ldots, f_T$. The temperature schedule is necessary because we only assume a cold start and do not assume strong convexity; in order for our variance-reduced SGLD to work, the initial starting point must be $\wt{O}_T (\nicefrac{1}{\sqrt T})$ rather than $\wt{O}_T(1)$ away from the minimum. The temperature schedule helps us get there by roughly halving the distance to the minimum each epoch; the step sizes are also halved at each epoch.  Moreover, one also cannot substitute a deterministic convex optimization algoritihm for initialization in our setting, since without strong convexity, deterministic convex optimization promises a point close in function value but not Euclidean distance. In contrast, our algorithm gives, with high probability, a point close enough in Euclidean distance if Assumption \ref{Assumption:Wass} holds.
\vspace{-3mm}
\begin{algorithm}[h]
\caption{Offline variance-reduced SGLD\label{alg:off-VRSGLD}} 
\textbf{Input:} $T\in \mathbb{N}$ and gradient oracles for functions $f_t: \R^d \rightarrow \mathbb{R}$, $1\le t\le T$.\\
\textbf{Input:} step size  $\eta$, batch size $b>0$,  $i_\mathrm{max}>0$, an initial point $\mathsf{X}^0\in \R^d$\\
\textbf{Output:} A sample $\mathsf{X}$
\begin{algorithmic}[1]
\State Set $\mathsf{X} \mapsfrom \mathsf{X}^0$ and set $\beta = \nicefrac{1}{T}$.
\Comment{Start at a high temperature, $T$. }
\While{$\beta<1$}
	 \State Run Algorithm~\ref{alg:SAGA} with step size $\nicefrac{\eta}{\beta T}$, batch size $b$, number of steps $i_{\max}$, initial point $\mathsf{X}$, and functions $\beta f_t$, $1\le t\le T$.
	 \State Set $\mathsf{X} \mapsfrom \mathsf{X}^\beta$, where $\mathsf{X}^\beta$ is the output of Algorithm~\ref{alg:SAGA}.
	 \State $\beta \mapsfrom \max\{2\beta,1\}$. \Comment{Double the temperature.}
\EndWhile
\State Return $\mathsf{X}$.
\end{algorithmic}
\end{algorithm}

\subsection{Proof overview of offline result}
For the offline problem, the desired result -- sampling from $\pi_T$ with TV error $\ep$ using $\tilde{O}(T) + \mathrm{poly}(d, L, C, \epsilon^{-1})\log_2(T)$ gradient evaluations -- is known either when we assume strong convexity, or we have a warm start. 
We show how to achieve the same additive bound without either assumption.

Without strong convexity, we do not have  access to a Lyapunov function which guarantees that the distance between the Markov chain and the mode $x^\star$ of the target distribution contracts at each step, even from a cold start. 
To get around this problem, we sample from a sequence of $\log_2(T)$ distributions $\pi_T^\beta\propto e^{-\beta \sumo tT f_t(x)}$, where the inverse ``temperature" $\beta$ doubles at each epoch from $\rc{T}$ to 1, causing the distribution $\pi_T^\beta$ to have a decreasing second moment and to become  more ``concentrated" about the mode $x^\star$ at each epoch.
This temperature schedule allows our algorithm to gradually approach the target distribution, even though our algorithm is initialized from a cold start $x^0$ which may be far from a sub-level set containing most of the target probability measure. The same martingale exit time argument as in the proof for the online problem shows that at the end of each epoch, the Markov chain is at a distance from $x^\star$ comparable to the (square root of the) second moment of the current distribution $\pi_T^\beta$. This provides a ``warm start" for the next distribution $\pi_T^{2\beta}$, and in this way our Markov chain approaches the target distribution $\pi_T^1$ in $\log_2(T)$ epochs.

The total number of gradient evaluations is therefore $T \log_2(T) + b \times i_{\mathrm{max}}$, since we only compute the full gradient at the beginning of each of the $\log_2(T)$ epochs, and then only use a batch size $b$ for the gradient steps at each of the $i_{\mathrm{max}}$ steps of the Markov chain.  As in the online case, $b$ and $i_{\mathrm{max}}$ are poly-logarithmic in $T$ and polynomial in the various parameters $d,L,C,  \ep^{-1}$, implying that the total number of gradient evaluations is $\tilde{O}(T) + \mathrm{poly}(d, C, \mathfrak{D}, \epsilon^{-1}, L)\log_2(T)$, in the offline setting where our goal is only to sample from $\pi_T^1$.

The proof of Theorem \ref{thm:main_offline} is similar to the proof of Theorem \ref{thm:os-main}, except for some differences as to how the stochastic gradients are computed and how one defines the functions ``$F_t$".
We define $F_t := \beta_t \sum_{k=1}^T f_k$, where 
$\beta_t=\begin{cases}
2^{t-1}/T,&0\le s\le \log_2(T)+1\\
1,&t=\ce{\log_2(T)}+1.
\end{cases}$.
We then show that for this choice of $F_t$ the offline assumptions, proof and algorithm are similar to those of the online case. 

%% file: offline.tex
\section{Proof of offline theorem (Theorem \ref{thm:main_offline})}

The proof of Theorem \ref{thm:main_offline} is similar to the proof of Theorem \ref{thm:os-main}, except for some key differences as to how the stochastic gradients are computed and how one defines the functions ``$F_t$".

We define $F_\beta := \beta F =\beta \sum_{k=1}^T f_k$, where the $\beta$'s will range over the sequence
\begin{align}
\beta_t=\begin{cases}
2^{t}/T,&0\le t\le \log_2(T)\\
1,&t=\ce{\log_2(T)}.
\end{cases}\end{align}
For this choice of $F_\beta$, the offline assumptions, proof and algorithm are similar to those of the online case. 

\paragraph{Differences in assumptions.}
We have that $F_\beta$ is $\beta T L$-smooth, which (except for Lemma \ref{lemma:variance}) is the only way in which Assumption \ref{a:smooth} is used in the proof of Theorem \ref{thm:os-main}.  

Moreover, Assumption \ref{a:wass2} for the offline case implies that $\pi_T^{\beta} \propto e^{-F_\beta}$ satisfies Assumption \ref{a:wass} with constants $C$ and $k$ for every $t$.  Since the minimizer $x_\beta^\star$ of $F_\beta$ does not change with $t$, $x_\beta^\star$  satisfies Assumption \ref{Assumption:MLE} with constant $\mathfrak{D}=0$.

\paragraph{Differences in algorithm.}
The step size used in Algorithm \ref{alg:off-VRSGLD} is $\fc{\eta}{\beta T}$, the same step size used in Algorithm \ref{alg:OSAGA}.   Thus, we note that Algorithm \ref{alg:off-VRSGLD} is similar to Algorithm \ref{alg:OSAGA} except for a few key differences:
 \begin{enumerate}
\item  The way in which the stochastic gradient $g_i^\beta$ is computed is different.  Specifically, in the offline algorithm our stochastic gradient is computed as
\begin{equation}
 g_i^\beta =
 s + \fc{\beta T}b \sum_{k\in S} (G_{\text{new}}^k - G^k).
 \end{equation}
 where $S$ is a multiset of size $b$ chosen with replacement from $\{1,\ldots, T\}$ (rather than from $\{1,\ldots, t\}$).
\item There are logarithmically many epochs.
\end{enumerate}

We now give the proof in some detail.

\ift{\nomenclature[3Gbeta]{$G_\beta$}{$G_\beta =\bc{
\forall i, \ve{X^\beta_i - x^\star}\le \fc{\fR}{\sqrt{\beta T}}
}$ (Chapter 3)}}
Letting $X^\beta_i$ be the iterates at inverse temperature $\beta$, define 
\begin{align}
G_\beta &=\bc{
\forall i, \ve{X^\beta_i - x^\star}\le \fc{\fR}{\sqrt{\beta T}}
}.
\end{align}

\begin{lem}[Analogue of Lemma~\ref{l:induct}]\label{l:induct-offline}
Assume that Assumptions~\ref{a:smooth} and~\ref{a:wass2} hold. Let $C=\pa{2+\rc k}\log\pf{A}{k^2}$, $C_1\ge C$, and suppose 
\begin{align}
\eta_0&\le \fc{\ep_2^2}{Ld + 4L^2\fR^2/b},\\
i_{\max} &\ge \fc{5C_1^2}{\eta_0\ep_2^2}.
\end{align}
Suppose $\ep_1>0$ is such that
\begin{align}
\label{e:induct-offline-ep}
\Pj\pa{\forall 0\le i\le i_{\max}, \ve{X^\beta_i - x^\star} \le \fc{\fR}{\sqrt{\beta T}} | \ve{X^\beta_0 - x^\star} \le \fc{C_1}{\sqrt{\beta T}}}
&\ge 1-\ep_1.
\end{align}
Suppose $\ve{X^\beta_0 - x^\star}\le \fc{2C_1}{\sqrt{\beta T}}$. Then
\begin{enumerate}
\item
$\ve{\cL(X^\beta)- \pi_T^\beta}_{TV}\le \ep_1+\ep_2$.
\item
For $i\in [i_{\max}]$ chosen at random, 
\begin{align}\Pj\pa{\ve{X^\beta_i - x^\star} \le \fc{C_1}{\sqrt{\beta T}}} 
&\ge 1-(\ep_1+\ep_2+Ae^{-kC_1}).
\end{align}
\end{enumerate}
\end{lem}
\begin{proof}
First we calculate the distance of the starting point from the stationary distribution, 
\begin{align}
W_2^2(\de_{X^\beta_0},\pi_T^\beta)&\le 
2\ve{X_0^\beta - x^\star}^2 + 2W_2^2(\de_{x^\star}, \pi_T^\beta)
\le \fc{8C_1^2}{\beta T} + \fc{2C^2}{\beta T} \le \fc{10C_1^2}{\beta T}.
\end{align}
Define a toy Markov chain coupled to $X^\beta_i$ as follows. Let $\wt{X}^\beta_0=X^\beta_0$ and
\begin{align}
\wt X_{i+1}^\beta &=
\begin{cases}
\wt X_i^\beta - \eta g_i^\beta + \sqrt{\eta} \xi_i,&\text{when }\ve{\wt X_j^\tau - x^\star}\le \fc{\fR}{\sqrt{\beta T}}\text{ for all }0\le j\le i\\
\wt X_i^\beta - \eta\beta\nb F(\wt X_i),&\text{otherwise.}
\end{cases}\label{e:toymc-offline}
\end{align}
By Lemma~\ref{l:var}, the variance of $g^\beta_i$ is at most $\fc{\beta^2T^2L^2}{b} \max_{0\le j\le i} \ve{\wt X^\beta_i - \wt X^\beta_j}^2$. If $\ve{X^\beta_i-x^\star}\le \fc{\fR}{\sqrt{\beta T}}$ for all $0\le i\le i_{\max}$, then $\ve{\wt X^\beta_i-\wt X^\beta_j}\le \fc{2\fR}{\sqrt{\beta T}}$ for all $0\le i,j\le i_{\max}$. Then we can apply Lemma~\ref{t:dmm} with $\ep=2\ep_2^2$, $L\mapsfrom L\beta T$, $\si^2 \le
\fc{(\beta T)^2 L^2}{b} \fc{4\fR^2}{\beta T}=
 \fc{4\beta T L^2\fR^2}{b}$, and $W_2^2(\mu_0,\pi)\le  \fc{10C_1^2}{\beta T}$. By Pinsker's inequality, for random $i\in [i_{\max}]$,
 \begin{align}
\ve{\cal L(\wt X^\beta_i)-\pi_T^\beta}_{\mathrm{TV}} &\le \sqrt{\rc2\mathrm{KL}(\wt \mu| \pi_\tau)} \leq \ep_2.
\end{align}
Under $G_\beta$, $X_i^\beta = \wt X_i^\beta$ for all $i\le i_{\max}$ and $s\le \tau$, so 
\begin{align}
 \|\mathcal{L}({X}^{\beta}_i) - \pi_T^\beta \|_{\mathrm{TV}} &\leq
\Pj(G_\beta^c)
+ \ve{\cal L(\wt X^\beta_i)-\pi_T^\beta}_{\mathrm{TV}} \leq
\ep_1+\ep_2.
 \label{e:induct1-offline}
\end{align}
This shows part 1. 

For part 2, note that by Assumption~\ref{a:conc},
\begin{align}
\Pj_{X\sim \pi_T^{\beta}}\ba{
\ve{X-x^\star}\ge \fc{C_1}{\sqrt{\beta T}}
}&\le 
Ae^{-kC_1}.\label{e:induct2-offline}
\end{align}
Combining~\eqref{e:induct1-offline} and~\eqref{e:induct2-offline} gives part 2.
\end{proof}

\begin{thm}[Theorem~\ref{thm:main_offline} with parameters]\label{t:main-param-offline}
Suppose that Assumptions~\ref{a:smooth} and~\ref{a:wass2} hold, with $L\ge 1$, $k\le 1$, and $\ve{X^0-x^\star}\le C$. 
Suppose Algorithm~\ref{alg:off-VRSGLD} is run with parameters $\eta_0,i_{\max}$ given by
\begin{align}
\ep_1&= \fc{\ep}{3 \ce{\log_2(T)+1}},\\
\label{e:main-C1-offline}
C_1 &=\pa{2+\rc k}\log\pf{A}{\ep_2k^2},\\
\fR &= 
\fc{10000 C_1\sqrt d}{\ep_1} \log\pa{\max\bc{L,C_1+\fD, \rc{\ep_1}}}
\\
\eta_0 &= \fc{\ep_1^{2}}{2L^2\fR^2},\\
i_{\max} &= 
\ce{\fc{5C_1^2}{\eta_0\ep_1^2}}
=\ce{\fc{10L^2\fR^2 C_1^2}{\ep_1^4}},
\end{align}
with any constant batch size $b\ge 4$. Then it outputs $X^1$ such that $X^1$ is a sample from $\wt\pi_T$ satisfying $\ve{\wt \pi_T-\pi_T}_{TV}\le \ep$, 
using $\wt O(T)+\poly\log(T)\poly(d,L,C,\ep^{-1})$ gradient evaluations.
\end{thm}
\begin{proof}[proof of Theorem \ref{thm:main_offline}]
The proof is similar to the proof of Theorem~\ref{t:main-param}, and we omit the details. 
We show by induction that 
\begin{align}
\Pj\pa{\ve{X^{\beta_s}_i-x^\star}\le \fc{\fR}{\sqrt{\beta_s T}}}&\ge 1-2s\ep_1.
\end{align}
The base case follows from $C\le C_1\le \fR$.
The induction step follows from noting first that
\begin{align}
\ve{X^{\beta_s}_i-x^\star}\le \fc{\fR}{\sqrt{\beta_s T}}
\implies 
\ve{X^{\beta_{s+1}}_0-x^\star}\le \fc{2\fR}{\sqrt{\beta_{s+1} T}},
\end{align}
noting that the conditions imply (for $\eta_\beta=\fc{\eta_0}{\sqrt{\beta T}}$, $r_t = \fc{\fR}{\sqrt{\beta T}}$, $S_t = 4\sqrt{\beta T} L\fR$, 
and $C_\xi = \sqrt{2d+8\log \pf{2i_{\max}}{\ep_1}}$) that 
\begin{align}
  \ep_1 & \ge i_{\max}\Bigg[
  \exp\pa{-\fc{(r_\beta^2 -  
   \fc{4C_1^2}{t+L_0/L} 
-  i [2\eta_t^2 (S_\beta^2 + L^2t^2 r_\beta^2) + \eta_\beta d])^2}{2i_{\max}( 2\eta_\beta S_\beta r_\beta + 2 \sqrt{\eta_\beta} C_\xi (r_\beta+\eta_\beta S_\beta + \eta_\beta Ltr_t) + \eta_\beta C_\xi^2 )^2}}\\
&\qquad \qquad   +\exp\pa{-\fc{C_\xi^2-d}8}\Bigg].
\end{align}
Then using Lemma~\ref{l:escape}, we get that~\eqref{e:induct-offline-ep} is satisfied with $\ep_1$, and the induction step follows from part 2 of Lemma~\ref{l:induct-offline}.

Finally, once we have $\ve{X^1_0-x^\star} \le \fc{\fR}{\sqrt{T}}$, the conclusion about $X^1$ follows from part 1 of Lemma~\ref{l:induct-offline}.
\end{proof}

%% file: simulation.tex
\section{Simulations}\label{s:exp}
We test our algorithm against other sampling algorithms on a synthetic dataset for logistic regression. 
The dataset consists of $T=1000$ data points in dimension $d=20$. We compare the marginal accuracies of the algorithms.

The data is generated as follows. First, $\te\sim N(0,I_d), b\sim N(0,1)$ are randomly generated.
For each $1\le t\le T$, a feature vector $x_t\in \R^d$ and output $y_t\in \{0,1\}$ are generated by 
\begin{align}
x_{t,i}&\sim \textrm{Bernoulli}\pf{s}{d}&1\le i&\le d,\\
y_t &\sim \textrm{Bernoulli}(\phi(\te^\top x_t+b)),
\end{align}
where the sparsity is $s=5$ in our simulations, and $\phi(x) = \rc{1+e^{-x}}$ is the logistic function.
We chose $x_t\in \{0,1\}^d$ because in applications, features are often indicators.

The algorithms are tested in an online setting as follows. At epoch $t$ each algorithm has access to $x_{s,i},y_s$ for $s\le t$, and attempts to generate a sample from the posterior distribution $p_t(\te) \propto e^{-\fc{\ve{\te}^2}2}e^{-\fc{b^2}2}\prodo st \phi(\te^\top x_t+b)$; the time is limited to $t=0.1$ seconds. We estimate the quality of the samples at $t=T=1000$, by saving the state of the algorithm at $t=T-1$, and re-running it 1000 times to collect 1000 samples. We replicate this entire simulation 8 times, and the marginal accuracies of the runs are given in Figure~\ref{f:1}. 

The marginal accuracy (MA) is a heuristic to compare accuracy of samplers (see e.g.  \cite{durmus2017convergence}, \cite{faes2011variational} and \cite{chopin2017leave}). 
The marginal accuracy between the measure $\mu$ of a sample and the target $\pi$ is $MA(\mu, \pi) := 1-\frac{1}{2d} \sum_{i=1}^d \|\mu_i - \pi_i\|_{\mathrm{TV}}$, where $\mu_i$ and $\pi_i$ are the marginal distributions of $\mu$ and $\pi$ for the coordinate $x_i$.  Since MALA is known to sample from the correct stationary distribution 
for the class of distributions analyzed in this paper, we let $\pi$ be the estimate of the true distribution obtained from 1000 samples generated from running MALA for a long time (1000 steps). We estimate the TV distance by the TV distance between the histograms when the bin widths are 0.25 times the sample standard deviation for the corresponding coordinate of $\pi$.

We compare our online SAGA-LD algorithm with SGLD, full and online Laplace approximation, P\'olya-Gamma, and MALA.
The Laplace method approximates the target distribution with a multivariate Gaussian distribution.  Here, one first finds the mode of the target distribution using a deterministic optimization technique and then computes the Hessian  $\nabla^2 F_t$  of the log-posterior at the mode.  The inverse of this Hessian is the covariance matrix of the Gaussian. In the online version of the algorithm, 
given in~\cite{chapelle2011empirical}, to speed up optimization, only a 	quadratic approximation (with diagonal Hessian) to the log-posterior is maintained.
The P\'olya-Gamma chain \cite{dumitrascu2018pg} is a Markov chain specialized to sample from the posterior for logistic regression. 
Note that in contrast, our algorithm works more generally for any smooth probability distribution over $\R^d$. 

Our results show that our online SAGA-LD algorithm is competitive with the best 
samplers for logistic regression, namely, the P\'olya-Gamma Markov  chain and the full Laplace approximation. 
We note that the full Laplace approximation requires optimizing a sum of $t$ functions, which has runtime that scales linearly with $t$ at each epoch, while our method only scales as $\text{polylog}(t)$.

The parameters are as follows. The step size at epoch $t$ is $\fc{0.1}{1+0.5t}$ for MALA, $\fc{0.01}{1+0.5t}$ for SGLD, and $\fc{0.05}{1+0.5t}$ for online SAGA-LD. A smaller step size must be used with SGLD because of the increased variance. For MALA, a larger step size can be used because the Metropolis-Hastings acceptance step ensures the stationary distribution is correct. The batch size for SGLD and online SAGA-LD is 64. The step sizes $\eta_0$ were chosen by hand from testing various values in the range from $0.001$ to $1.0$.
We found the reset step of our online SAGA-LD algorithm, and the random number of steps, to be unnecessary in practice, so the results are reported for our online SAGA-LD algorithm without these features.
The experiments were run on Fujitsu CX2570 M2 servers with dual, 14-core 2.4GHz Intel Xeon E5 2680 v4 processors with 384GB RAM running the Springdale distribution of Linux.

\begin{figure}
\iftoggle{thesis}{\includegraphics[scale=0.5,valign=c]{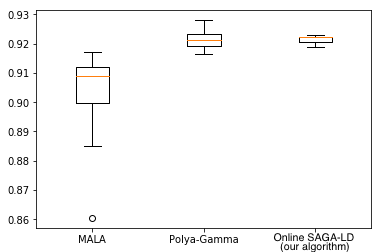}}{\includegraphics[scale=0.6,valign=c]{box2.png}}\begin{tabular}{|c|c|}
\hline 
Algorithm & Mean marginal accuracy\tabularnewline
\hline 
\hline 
SGLD & 0.442\tabularnewline
\hline 
Online Laplace & 0.571\tabularnewline
\hline 
MALA & 0.901\tabularnewline
\hline 
Polya-Gamma & 0.921\tabularnewline
\hline 
\textbf{Online SAGA-LD} & \multirow{2}{*}{0.921}\tabularnewline
\textbf{(our algorithm)} & \tabularnewline
\hline
Full Laplace & 0.924 \tabularnewline
\hline 
\end{tabular}
\caption{Marginal accuracies of 6 different sampling algorithms on online logistic regression, with $T=1000$ data points, dimension $d=20$, and time 0.1 seconds, averaged over 8 runs. SGLD and online Laplace perform much worse and are not pictured.}\label{f:1}
\end{figure}

%
%
%
%
%
%
%

%% file: conclusions.tex
\section{Discussion and future work}
\label{s:conclusion}

\ifarxiv{In this paper we obtain logarithmic-in-$T$ bounds at each epoch when sampling from a sequence of log-concave distributions $\pi_t \propto e^{-\sum_{k=0}^tf_k}$, improving on previous results which are linear-in-$T$ in the online setting.    Since we do not assume the $f_t$'s are strongly convex, we also obtain bounds which have an improved dependence on $T$  for a wider range of applications including Bayesian logistic regression.}{}


\paragraph{Comparison to using a regularizer.}
Recall that one issue in proving Theorem~\ref{thm:os-main} is that we don't assume the $f_t$ are strongly convex. One way to get around this is to add a strongly convex regularizer, and use existing results for Langevin in the strongly convex case. In the online case, one would have to add $\ep t ||x-\hat x_t||^2$ to the objective, where $\hat x_t$ is an estimate of the mode $x_t^\star$. Assuming we have such an estimate, using results on Langevin for strong convexity, to get $\epsilon$ TV-error, we also require $\tilde O\prc{\ep^6}$ steps per iteration. (Specifically, use \cite[Corollary 22]{durmus2018analysis}, with strong convexity $m=\ep t$ to get that $\tilde O\prc{\epsilon^3}$ iterations are required to get KL-error $\ep$, and apply Pinsker's inequality.)

%

\paragraph{Preconditioning.} 
Note our result does not hold if the covariance matrix of the $u_t$'s distribution becomes much more ill-conditioned over time, as is the case in certain Thompson sampling applications \cite{russo2017tutorial}. 

We would like to obtain similar bounds under more general assumptions where the covariance matrix could change at each epoch and be ill-conditioned.  This type of distribution arises in reinforcement learning applications such as Thompson sampling \cite{dumitrascu2018pg}, where the data is determined by the user's actions.  If the user favors actions in certain ``optimal" directions, in some cases the distribution may have a much smaller covariance in those directions than in other directions, causing the covariance matrix of the target distribution to become more ill-conditioned over time.  


\paragraph{Improved bounds for strongly convex functions.} Suppose that we dropped the requirement of independence. Note that if we use SAGA-LD with the last sample from the previous epoch, we have a warm start for the previous distribution, and would be able to achieve TV error that decreases as $T$ with  $\wt O_T(1)$ time per epoch. It seems possible to reduce the TV error to $O\pf{\ep}{t^{\rc 6}}$ this way, and possibly to $O\pf{\ep}{t^{\rc 4}}$ with stronger drift assumptions. These guarantees may also extend to subexponential distributions.

\paragraph{Distributions over discrete spaces.} 
There has been work on stochastic methods in the setting of discrete variables \cite{pmlr-v80-desa18a} that could potentially be used to develop analogous theory in the discrete case.

\paragraph{Non-compact distributions} One can also consider the problem of sampling from log-densities which are a sum of $T$ functions with compact support (online sampling from such distributions was considered in \cite{narayanan2013efficient}, but their running time bounds are not logarithmic in $T$  at each epoch).   One cannot directly apply our results to compactly supported log-densities, since they do not satisfy our Lipschitz gradient assumption (Assumption \ref{a:smooth}).  At the very least we would have to modify our algorithm, for example by rejecting steps proposed by our algorithm that would otherwise cause the Markov chain to leave the support of the target distribution.  A more challenging issue would be that restricting the distribution to a compact support can cause the distribution’s covariance matrix to become increasingly ill-conditioned as the number of functions $t$ increases, even if the support is convex.  To get around this problem we would need to modify our algorithm by including an adaptive pre-conditioner which changes along with the changing target distribution.

\paragraph{Necessity of drift condition (Assumption \ref{a:drift}).}
Since we do not assume that the individual functions $f_k$ are strongly convex, the mode (or, alternatively, the mean) of the target distribution cannot be controlled by the mode (or mean) of the individual functions.  For instance, in logistic regression, all of the individual functions have “mode” at $\pm \infty$ in the direction of the data vector.  Therefore, unlike in the strongly convex case, a condition on the mode of each individual function $f_k$ does not suffice for many non-strongly convex applications including logistic regression.  Rather, the mode depends on the probability distribution from which the individual functions are drawn.  We show that Assumption \ref{a:drift} holds in Section \ref{sec:Bayesian_summary} for the special case of Bayesian logistic regression, and give more general conditions for when Assumption \ref{a:drift} holds in Theorem \ref{thm:general-conc}.

%% file: simple.tex
\section{A simple example where our assumptions hold} \label{sec:simple_example}
As a simple example to motivate our assumptions, we consider the Bayesian linear regression model $y_t =  z_t^\top \theta_0 + w_t$, where $y_t \in \mathbb{R}^1$ is the dependent variable, $z_t \in \mathbb{R}^d$ the independent variable, and   $w_t \sim N(0,1)$ the unknown noise term.  The Bayesian posterior distribution for the coefficient $\theta_0$ is
$\pi_t(\theta) \propto e^{-\sum_{k=1}^t f_k(\theta)} = e^{-[\theta- \mu]^\top \Sigma^{-1} [\theta- \mu]}$
where $f_k(\theta) = (y_k - z_k \theta)^2$ for each $k$, $\Sigma^{-1} = \sum_{k=1}^T z_k z_k^\top$ and $\mu = \Sigma^{\nicefrac{1}{2}} \sum_{k=1}^T y_k z_k$.  Hence, the posterior $\pi_t$ has distribution $N(\mu, \Sigma)$.  While computing $\Sigma$ requires at least $T\times d^2$, computing a stochastic gradient with batch size $b$ requires $d \times b$ operations.  Therefore, one can hope to sample in fewer than $T\times d^2$ operations (we prove this in Theorem \ref{thm:os-main}).  

We now show that our assumptions hold for this example. For simplicity, we assume that the dimension $d=1$, $z_t =1$ for all $t$, and assume an improper ``flat" prior, that is, $f_0=0$.  At each epoch $t\in \{1,\ldots,T\}$, the Bayesian posterior distribution for the coefficient $\theta_0$ is $\pi_t(\theta) \propto e^{-\sum_{k=1}^t f_k(\theta)}$, which a simple computation shows is the normal distribution with mean $\theta_0  +  \frac{\sum_{k=0}^t w_k}{t}$ and variance $\frac{1}{2t} \leq \frac{1}{t+1}$.   Thus, Assumption \ref{Assumption:LipschitzG} is satisfied with $L=1$ and Assumption \ref{Assumption:Wass} is satisfied with $C=2$.  To verify Assumption \ref{Assumption:MLE}, we note that $x_t^\star = \frac{\sum_{k=1}^t w_k}{t}$, and thus $x_t^\star \sim N(0, \frac{1}{t})$.   We can then apply Gaussian concentration inequalities to show that $\mathfrak{D} = 4 \log^{\frac{1}{2}}(\frac{\log(T)}{\delta})$ with probability at least $1-\delta$.

%% file: hardness.tex
\section{Hardness} \label{sec:Hardness}\label{s:hard}

\paragraph{Hardness of optimization with stochastic gradients.}
The authors of \cite{agarwal2009information} consider the problem of optimizing an $L$-Lipschitz function $F:\mathcal{K} \rightarrow \mathbb{R}$ on a convex body $K$ contained in an $\ell_\infty$ ball of radius $r>0$. Given an initial point in $\mathcal{K}$ and access to a first-order stochastic gradient oracle with variance $\si^2$, they show that any optimization method, given a worst-case initial point in $\mathcal{K}$, requires at least $\Omega(\frac{L^2\si^2 d}{\delta^2})$ calls to the stochastic gradient oracle to obtain a random point $\hat{x}$ such that $\mathbb{E}[F(\hat{x}) - F(x^\star)]\leq \delta$.
\vspace{-2mm}
\paragraph{Hardness in our setting.} 
What is the minimum number of gradient evaluations required to sample from a target distribution satisfying Assumptions \ref{Assumption:LipschitzG}--\ref{Assumption:MLE} with fixed TV error $\epsilon >0$, given only access to the gradients $\nb f_k$, $0\le k\le T$?  In this section we show (informally) by counterexample that one needs to compute at least $\Omega(T)$ gradients to sample with TV error $\epsilon \leq \frac{1}{20}$. As a counterexample, consider the Bayesian linear regression posterior considered in Section \ref{sec:simple_example}, with $d=1$.  Suppose that one only computes stochastic gradients using gradients with index in a random set $S_i = \{\tau_1,\ldots, \tau_{\frac{T}{2}}\}$, of size $\frac{T}{2}$, where each element of $S_i$ is chosen independently from the uniform distribution on $\{1,\ldots, T\}$.  Then the mean of these stochastic gradients (conditioned on the subset $S_i$) are gradients of a function $-\log(\hat{\pi}^{(i)})$, for which $\hat{\pi}^{(i)}$ is the density of the normal distribution $N(\mu_i,  \frac{1}{2t})$, where the mean $\mu_i = \frac{\sum_{k\in S_i} w_k}{t} \sim N(0, \frac{1}{t})$ is itself (conditional on $S_i$) a random variable.   Now consider two independent random subsets $S_1$ and $S_2$ with corresponding distributions $\hat{\pi}^{(1)}$ and $\hat{\pi}^{(2)}$.  The means of the distributions $\hat{\pi}^{(1)}$ and $\hat{\pi}^{(2)}$ (conditional on $S_1$ and $S_2$) are independent random variables $\mu_1, \mu_2 \sim  N(0, \frac{1}{t})$.  Hence, the difference in their means $\mu_1 - \mu_2 \sim  N(0, \frac{2}{t})$ is normally distributed with standard deviation $\frac{\sqrt{2}}{\sqrt{t}}$.  Thus, with probability at least $\frac{1}{2}$, we have $|\mu_1 - \mu_2| \geq \frac{1}{\sqrt{t}}$.  Therefore, since (conditional on $S_1, S_2$) we have $\hat{\pi}^{(i)} \sim N(\mu_i,  \frac{1}{2t})$ for $i \in \{1,2\}$, we must have that $\|\hat{\pi}^{(1)}- \hat{\pi}^{(2)}\|_{\mathrm{TV}} \geq \frac{1}{10}$ whenever $|\mu_1 - \mu_2| \geq \frac{1}{\sqrt{t}}$. That is, $\|\hat{\pi}^{(1)}- \hat{\pi}^{(2)}\|_{\mathrm{TV}} \geq \frac{1}{10}$ occurs with probability at least $\frac{1}{2}$.  Therefore, one cannot hope to sample from $\pi_T$ with TV error $\epsilon< \frac{1}{20}$ by using the information from only $\frac{T}{2}$ gradients.  One therefore needs to compute at least $\Omega(T)$ gradients to sample from $\pi_T$ with TV error $\epsilon < \frac{1}{20}$.

%% file: inequalities.tex
\iftoggle{thesis}{}{\section{Miscellaneous inequalities}

We give some inequalities used in the proofs in Section~\ref{s:bayes}.}

\begin{lem}\label{lem:azuma-d}
Suppose that $X_t$ are a sequence of random variables in $\mathbb{R}^d$ and for each $t$, $\ve{X_t - \E[X_t|X_{1:t-1}]}_{\iy}\leq M$ (with probability 1). Let $S_T= \sumo tT \E [X_t|X_{1:t-1}]$ (a random variable depending on $X_{1:T}$). Then
\begin{align}
\Pj\pa{
\ve{\sumo tT  X_t - S_t}_2 \geq c
} & \leq 2d e^{-\fc{c^2T}{2M^2d}}.
\end{align}
\end{lem}
\begin{proof}
By Azuma's inequality, for each $1\leq j\leq d$, 
\begin{align}
\Pj \pa{\ab{\sumo tT (X_t)_j - (S_t)_j} \geq c}& \leq 2e^{-\fc{c^2T}{2M^2}}.
\end{align}
By a union bound,
\begin{align}
\Pj\pa{
\ve{\sumo tT  X_t - S_t}_2 \geq c
} 
&\leq \sumo jd
\Pj \pa{\ab{\sumo tT (X_t)_j - (S_t)_j} \geq \fc{c}{\sqrt d}}\leq
 2de^{-\fc{c^2T}{2M^2d}}.
\end{align}
\end{proof}

\begin{lemma}\label{l:subexp-implies-2m}\label{l:exp-2m}
Suppose that $\pi$ is a distribution with $\Pj_{\te\sim \pi}(\ve{\te-\te_0}\ge \ga)\le Ae^{-k\ga}$, for some $\te_0$. 
Then 
\begin{equs}
\E_{\te\sim \pi}[\ve{\te-\te_0}^2]&\le \pa{2+\rc k}\log\pf{A}{k^2}.
\end{equs}
\end{lemma}
\begin{proof}
Without loss of generality, $\te_0=0$. 
Then
\begin{align}
\E_{\te\sim \pi}[\ve{\te}^2] &= \int_0^\iy 2\ga \Pj_{\te\sim \pi}(\ve{\te}\ge \ga)\,d\ga\\
&\le \ga_0 + \int_{\ga_0}^\iy2\ga \Pj_{\te\sim \pi}(\ve{\te}\ge \ga)\,d\ga\\
&\le \ga_0 + \int_{\ga_0}^\iy2\ga Ae^{-k\ga}\,d\ga&\text{by assumption}&\\
&= \ga_0 + A\pa{-\fc{2\ga}{k}e^{-k\ga}\Big|^\iy_{\ga_0} - \int_{\ga_0}^\iy -\fc{2}k e^{-k\ga}\,d\ga}&\text{integration by parts}&\\
&=A\pa{\fc{2\ga_0}{k}e^{-k\ga_0} + \fc{2}{k^2} e^{-k\ga_0}}.
\end{align}
Set $\ga_0=\fc{\log\pf{A}{k^2}}k$. Then this is $\le \pa{2+\rc k} \log\pf{A}{k^2}$, as desired.
\end{proof}

%% file: arxiv.bbl
\newcommand{\etalchar}[1]{$^{#1}$}
\begin{thebibliography}{DMHW{\etalchar{+}}12}

\bibitem[AC93]{albert1993bayesian}
James~H Albert and Siddhartha Chib.
\newblock Bayesian analysis of binary and polychotomous response data.
\newblock {\em Journal of the American statistical Association},
  88(422):669--679, 1993.

\bibitem[ADH10]{andrieu2010particle}
Christophe Andrieu, Arnaud Doucet, and Roman Holenstein.
\newblock Particle {M}arkov chain {M}onte {C}arlo methods.
\newblock {\em Journal of the Royal Statistical Society: Series B (Statistical
  Methodology)}, 72(3):269--342, 2010.

\bibitem[AWBR09]{agarwal2009information}
Alekh Agarwal, Martin~J Wainwright, Peter~L Bartlett, and Pradeep~K Ravikumar.
\newblock Information-theoretic lower bounds on the oracle complexity of convex
  optimization.
\newblock In {\em Advances in Neural Information Processing Systems}, pages
  1--9, 2009.

\bibitem[BBW{\etalchar{+}}13]{broderick2013streaming}
Tamara Broderick, Nicholas Boyd, Andre Wibisono, Ashia~C Wilson, and Michael~I
  Jordan.
\newblock Streaming variational {B}ayes.
\newblock In {\em Advances in Neural Information Processing Systems}, pages
  1727--1735, 2013.

\bibitem[BDT16]{barber2016laplace}
Rina~Foygel Barber, Mathias Drton, and Kean~Ming Tan.
\newblock Laplace approximation in high-dimensional {B}ayesian regression.
\newblock In {\em Statistical Analysis for High-Dimensional Data}, pages
  15--36. Springer, 2016.

\bibitem[BNJ03]{blei2003latent}
David~M Blei, Andrew~Y Ng, and Michael~I Jordan.
\newblock Latent {D}irichlet allocation.
\newblock {\em Journal of machine Learning research}, 3(Jan):993--1022, 2003.

\bibitem[CB18]{campbell2018bayesian}
Trevor Campbell and Tamara Broderick.
\newblock Bayesian coreset construction via greedy iterative geodesic ascent.
\newblock In {\em International Conference on Machine Learning}, pages
  697--705, 2018.

\bibitem[CB19]{campbell2017automated}
Trevor Campbell and Tamara Broderick.
\newblock Automated scalable {B}ayesian inference via {H}ilbert coresets.
\newblock {\em The Journal of Machine Learning Research}, 20(1):551--588, 2019.

\bibitem[CBL06]{cesa2006prediction}
Nicolo Cesa-Bianchi and G{\'a}bor Lugosi.
\newblock {\em Prediction, learning, and games}.
\newblock Cambridge university press, 2006.

\bibitem[CFM{\etalchar{+}}18]{chatterji2018theory}
Niladri Chatterji, Nicolas Flammarion, Yian Ma, Peter Bartlett, and Michael
  Jordan.
\newblock On the theory of variance reduction for stochastic gradient {M}onte
  {C}arlo.
\newblock In Jennifer Dy and Andreas Krause, editors, {\em Proceedings of the
  35th International Conference on Machine Learning}, volume~80 of {\em
  Proceedings of Machine Learning Research}, pages 764--773,
  Stockholmsm\"assan, Stockholm Sweden, 10--15 Jul 2018. PMLR.

\bibitem[CL11]{chapelle2011empirical}
Olivier Chapelle and Lihong Li.
\newblock An empirical evaluation of {T}hompson sampling.
\newblock In {\em Advances in neural information processing systems}, pages
  2249--2257, 2011.

\bibitem[CR{\etalchar{+}}17]{chopin2017leave}
Nicolas Chopin, James Ridgway, et~al.
\newblock Leave pima indians alone: binary regression as a benchmark for
  {B}ayesian computation.
\newblock {\em Statistical Science}, 32(1):64--87, 2017.

\bibitem[DCWY18]{dwivedi2018log}
Raaz Dwivedi, Yuansi Chen, Martin~J Wainwright, and Bin Yu.
\newblock Log-concave sampling: {M}etropolis-{H}astings algorithms are fast!
\newblock In {\em Proceedings of the 2018 Conference on Learning Theory, PMLR
  75}, 2018.

\bibitem[DDFMR00]{doucet2000rao}
Arnaud Doucet, Nando De~Freitas, Kevin Murphy, and Stuart Russell.
\newblock Rao-{B}lackwellised particle filtering for dynamic {B}ayesian
  networks.
\newblock In {\em Proceedings of the Sixteenth conference on Uncertainty in
  artificial intelligence}, pages 176--183. Morgan Kaufmann Publishers Inc.,
  2000.

\bibitem[DFE18]{dumitrascu2018pg}
Bianca Dumitrascu, Karen Feng, and Barbara~E Engelhardt.
\newblock {PG-TS}: Improved {T}hompson sampling for logistic contextual
  bandits.
\newblock In {\em Advances in neural information processing systems}, 2018.

\bibitem[DMHW{\etalchar{+}}12]{del2012concentration}
Pierre Del~Moral, Peng Hu, Liming Wu, et~al.
\newblock On the concentration properties of interacting particle processes.
\newblock {\em Foundations and Trends{\textregistered} in Machine Learning},
  3(3--4):225--389, 2012.

\bibitem[DMM19]{durmus2018analysis}
Alain Durmus, Szymon Majewski, and B{\l}a{\.z}ej Miasojedow.
\newblock Analysis of {L}angevin {M}onte {C}arlo via convex optimization.
\newblock {\em Journal of Machine Learning Research}, 20(73):1--46, 2019.

\bibitem[DMS17]{durmus2017convergence}
Alain Durmus, Eric Moulines, and Eero Saksman.
\newblock On the convergence of {H}amiltonian {M}onte {C}arlo.
\newblock {\em arXiv preprint arXiv:1705.00166}, 2017.

\bibitem[DRW{\etalchar{+}}16]{dubey2016variance}
Kumar~Avinava Dubey, Sashank~J Reddi, Sinead~A Williamson, Barnabas Poczos,
  Alexander~J Smola, and Eric~P Xing.
\newblock Variance reduction in stochastic gradient {L}angevin dynamics.
\newblock In {\em Advances in neural information processing systems}, pages
  1154--1162, 2016.

\bibitem[DSCW18]{pmlr-v80-desa18a}
Chris De~Sa, Vincent Chen, and Wing Wong.
\newblock Minibatch {G}ibbs sampling on large graphical models.
\newblock In Jennifer Dy and Andreas Krause, editors, {\em Proceedings of the
  35th International Conference on Machine Learning}, volume~80 of {\em
  Proceedings of Machine Learning Research}, pages 1165--1173,
  Stockholmsm\"assan, Stockholm Sweden, 10--15 Jul 2018. PMLR.

\bibitem[FE15]{filippone2015enabling}
Maurizio Filippone and Raphael Engler.
\newblock Enabling scalable stochastic gradient-based inference for gaussian
  processes by employing the unbiased linear system solver (ulisse).
\newblock In {\em International Conference on Machine Learning}, pages
  1015--1024, 2015.

\bibitem[FKL{\etalchar{+}}18]{foster2018logistic}
Dylan~J Foster, Satyen Kale, Haipeng Luo, Mehryar Mohri, and Karthik Sridharan.
\newblock Logistic regression: The importance of being improper.
\newblock {\em Proceedings of Machine Learning Research vol}, 75:1--42, 2018.

\bibitem[FOW11]{faes2011variational}
Christel Faes, John~T Ormerod, and Matt~P Wand.
\newblock Variational {B}ayesian inference for parametric and nonparametric
  regression with missing data.
\newblock {\em Journal of the American Statistical Association},
  106(495):959--971, 2011.

\bibitem[GDM{\etalchar{+}}17]{giraud2017nonasymptotic}
Fran{\c{c}}ois Giraud, Pierre Del~Moral, et~al.
\newblock Nonasymptotic analysis of adaptive and annealed {F}eynman--{K}ac
  particle models.
\newblock {\em Bernoulli}, 23(1):670--709, 2017.

\bibitem[GLR18]{ge2018simulated}
Rong Ge, Holden Lee, and Andrej Risteski.
\newblock Simulated tempering {L}angevin {M}onte {C}arlo {II}: An improved
  proof using soft {M}arkov chain decomposition.
\newblock {\em arXiv preprint arXiv:1812.00793}, 2018.

\bibitem[HAK07]{hazan2007logarithmic}
Elad Hazan, Amit Agarwal, and Satyen Kale.
\newblock Logarithmic regret algorithms for online convex optimization.
\newblock {\em Machine Learning}, 69(2-3):169--192, 2007.

\bibitem[Haz16]{hazan2016introduction}
Elad Hazan.
\newblock Introduction to online convex optimization.
\newblock {\em Foundations and Trends{\textregistered} in Optimization},
  2(3-4):157--325, 2016.

\bibitem[HCB16]{huggins2016coresets}
Jonathan Huggins, Trevor Campbell, and Tamara Broderick.
\newblock Coresets for scalable {B}ayesian logistic regression.
\newblock In {\em Advances in Neural Information Processing Systems}, pages
  4080--4088, 2016.

\bibitem[HKL14]{hazan2014logistic}
Elad Hazan, Tomer Koren, and Kfir~Y Levy.
\newblock Logistic regression: Tight bounds for stochastic and online
  optimization.
\newblock In {\em Conference on Learning Theory}, pages 197--209, 2014.

\bibitem[KM15]{koltchinskii2015bounding}
Vladimir Koltchinskii and Shahar Mendelson.
\newblock Bounding the smallest singular value of a random matrix without
  concentration.
\newblock {\em International Mathematics Research Notices},
  2015(23):12991--13008, 2015.

\bibitem[Men14]{mendelson2014learning}
Shahar Mendelson.
\newblock Learning without concentration.
\newblock In {\em Conference on Learning Theory}, pages 25--39, 2014.

\bibitem[NDH{\etalchar{+}}17]{nagapetyan2017true}
Tigran Nagapetyan, Andrew~B Duncan, Leonard Hasenclever, Sebastian~J Vollmer,
  Lukasz Szpruch, and Konstantinos Zygalakis.
\newblock The true cost of stochastic gradient {L}angevin dynamics.
\newblock {\em arXiv preprint arXiv:1706.02692}, 2017.

\bibitem[Nic12]{nickl2012statistical}
Richard Nickl.
\newblock Statistical theory.
\newblock {\em Statistical Laboratory, Department of Pure Mathematics and
  Mathematical Statistics, University of Cambridge}, 2012.

\bibitem[NR17]{narayanan2013efficient}
Hariharan Narayanan and Alexander Rakhlin.
\newblock Efficient sampling from time-varying log-concave distributions.
\newblock {\em The Journal of Machine Learning Research}, 18(1):4017--4045,
  2017.

\bibitem[RVRK{\etalchar{+}}18]{russo2017tutorial}
Daniel~J Russo, Benjamin Van~Roy, Abbas Kazerouni, Ian Osband, Zheng Wen,
  et~al.
\newblock A tutorial on {T}hompson sampling.
\newblock {\em Foundations and Trends{\textregistered} in Machine Learning},
  11(1):1--96, 2018.

\bibitem[WPB11]{wang2011online}
Chong Wang, John Paisley, and David Blei.
\newblock Online variational inference for the hierarchical {D}irichlet
  process.
\newblock In {\em Proceedings of the Fourteenth International Conference on
  Artificial Intelligence and Statistics}, pages 752--760, 2011.

\bibitem[WT11]{welling2011bayesian}
Max Welling and Yee~W Teh.
\newblock Bayesian learning via stochastic gradient {L}angevin dynamics.
\newblock In {\em Proceedings of the 28th International Conference on Machine
  Learning (ICML-11)}, pages 681--688, 2011.

\bibitem[Zin03]{zinkevich2003online}
Martin Zinkevich.
\newblock Online convex programming and generalized infinitesimal gradient
  ascent.
\newblock In {\em Proceedings of the 20th International Conference on Machine
  Learning (ICML-03)}, pages 928--936, 2003.

\end{thebibliography}
